\documentclass[opre,nonblindrev]{informs3aa}
\OneAndAHalfSpacedXI

\usepackage{times}
\usepackage{amsmath,amssymb}
\usepackage{graphicx}
\usepackage{needspace}
\usepackage{placeins}
\usepackage{url}
\usepackage{natbib}
\bibpunct[, ]{(}{)}{,}{a}{}{,}
\def\bibfont{\small}

\usepackage[colorlinks=true,breaklinks=true,bookmarks=false,
  urlcolor=blue,citecolor=blue,linkcolor=blue]{hyperref}

\theoremstyle{TH}
\newtheorem{theorem}{Theorem}
\newtheorem{proposition}[theorem]{Proposition}
\newtheorem{corollary}[theorem]{Corollary}
\newtheorem{lemma}[theorem]{Lemma}
\let\INFORMSproof\proof
\let\INFORMSendproof\endproof
\renewenvironment{proof}[1][Proof]
  {\INFORMSproof{#1.}}
  {\unskip\nobreak\hfill\ensuremath{\square}\INFORMSendproof}
\EquationsNumberedThrough

\newcommand{\IdealPathCumulativeProofSpace}{\Needspace{6\baselineskip}}

\begin{document}

\RUNAUTHOR{Xiao and Jia}
\RUNTITLE{Ideal Paths for Approximating Logistic Gradient Descent Trajectories at Large Initialization}
\TITLE{Ideal Paths for Approximating Logistic Gradient Descent Trajectories at Large Initialization}

\ARTICLEAUTHORS{\AUTHOR{Junjie Xiao$^{1,2}$, Huiwen Jia$^{2}$}
\AFF{$^{1}$Department of Mathematics, Peking University\\
$^{2}$Department of Industrial Engineering and Operations Research, University of California, Berkeley}
}

\ABSTRACT{

Modern training on a new task often starts from a previously trained model  rather than from scratch, raising the question of how this initialization affects the subsequent training trajectory. Classical implicit-bias results characterize the direction selected by prolonged training, but this direction alone does not provide information regarding the intermediate behavior.  We address this question through a geometric approximation of full-batch logistic gradient descent (GD) trajectories on strictly linearly separable data, with large initialization of scale $R$ motivated by prior training.  From any limiting normalized initial position, we use minimum-norm projection rules to construct a unique continuous \textit{ideal path} consisting of finitely many linear segments. The path has two stages: negative-margin correction followed by minimum-margin growth. We prove that, after an explicit two-stage time reparameterization, the fixed-step GD trajectory divided by $R$ converges uniformly to this path on every fixed parameter interval as $R\to\infty$. Further, our quantitative error bounds account for initialization perturbations and the transition between stages. This approximation provides asymptotic formulas for peak evaluation loss and cumulative training loss. In particular, peak evaluation loss can grow linearly in $R$ even when both endpoint losses tend to zero. The cumulative losses in the correction and margin-growth stages, normalized by $R^2$ and $R$, respectively, converge to explicit limits. Experiments on controlled geometries and fixed image features complement our theoretical results.

}


\maketitle

\section{Introduction}
\label{sec:introduction}

Training on new tasks nowadays often starts from a previously trained model rather than from scratch. This raises the question of how initialization shapes the subsequent training trajectory. For separable logistic regression, the parameter norm diverges while its direction approaches the Euclidean maximum-margin separator \citep{soudry2018implicit}; optimization geometry and model homogeneity also shape asymptotic margin selection \citep{gunasekar2018geometry,lyu2020margin}. Reusing such parameters on another task provides a natural motivation for studying gradient descent at large initialization. Understanding the intervening trajectory can be of vital importance because evaluation loss can increase substantially during training even when it is small at both endpoints. However, the limiting direction alone does not determine this trajectory or the losses encountered along it.

As a first step toward addressing this question, we study full-batch
logistic gradient descent on fixed, strictly linearly separable data,
starting from $Ru_0+e_R$ as $R\to\infty$ and $\|e_R\|/R\to0$.
Here, $R>0$ controls the initialization scale.
The limiting normalized initial position $u_0$ is arbitrary and may
have negative margins.
We normalize the trajectory by $R$, rather than by its evolving norm,
to retain changes in both norm and direction.

Characterizing this trajectory requires accounting for changes in both
gradient direction and magnitude.
For logistic loss, large negative margins yield nearly constant
gradient contributions, whereas large positive margins yield
exponentially small ones.
The dominant samples can change as margin signs and orderings change,
and their relative contributions at zero margins and ties must be
determined in the limit.
Moreover, negative-margin correction and subsequent margin growth
occur on different time scales, requiring a stage-dependent
rescaling of iteration time.

We introduce an \emph{ideal path}, a
geometric approximation to the normalized GD trajectory at large
initialization.
Starting from $u_0$, the path first corrects any negative margins,
moving toward the region where all margins are nonnegative.
As positive margins grow, samples with the smallest margins govern
its subsequent motion.
Within each stage, the path follows straight segments and can turn
when margin signs or the set of smallest-margin samples change.
A time reparameterization accounts for the different rates of motion
in the two stages.
The path can be constructed directly from the initial position and
weighted data, independently of a GD simulation.

\subsection{Contributions}

Our main contributions are as follows.
\begin{itemize}

    \item We give a finite geometric construction of an \emph{ideal path}
    for logistic gradient descent at large initialization. The path
    consists of two stages: it first corrects negative margins and then
    grows the minimum margin. In both stages, the trajectory is
    piecewise linear and is determined by a minimum-norm projection rule
    (see Section~2).

    \item We show that this geometric path approximates the actual
    fixed-step GD trajectory. Under a two-stage time reparameterization,
    the GD trajectory normalized by $R$ converges uniformly to the ideal
    path on every fixed parameter interval, with error
    \[
        O\!\left(\frac{\|e_R\|}{R} + R^{-1/2}\right).
    \]
    The bound keeps track of perturbations in the initialization as well
    as the location and timing of the transition between the two stages.
    Under a standard smoothness-based step-size condition, the
    $R^{-1/2}$ term improves to $\log(2+R)/R$ (see Section~3).

    \item We use the path approximation to characterize quantities that
    depend on the full training trajectory rather than only its endpoint.
    In particular, evaluation loss can grow to order $R$ at an
    intermediate point even when both endpoint losses vanish. We also
    obtain different leading scales for cumulative training loss in the
    two stages: $R^2$ during negative-margin correction and $R$
    afterwards (see Section~4).

    \item We compare these predictions with fixed-step GD on controlled
    examples, random separable instances, and fixed image features. As
    the initialization scale increases, the observed trajectories
    approach the ideal paths, and the experiments also match the
    predicted behavior of peak evaluation loss and cumulative training
    loss (see Section~5).

\end{itemize}

\subsection{Related Work}
\label{sec:related-work}

Our work is closely related to three streams of literature: implicit bias and margin geometry, initialization scale and training trajectories, and sequential training and path observables.

\paragraph{Implicit bias and margin geometry.}
\citet{soudry2018implicit} establish convergence toward the Euclidean maximum-margin direction for logistic regression on separable data. \citet{nacson2019convergence} analyze how loss tails and step sizes influence margin convergence. Optimization geometry and model homogeneity shape implicit bias \citep{gunasekar2018geometry,lyu2020margin}. \citet{ji2021primaldual} characterize implicit bias through a smoothed-margin dual problem, relating GD to mirror descent in dual variables. For a broad class of convex decreasing losses, \citet{ji2020regularization} connect GD with norm-constrained risk minimization: both paths approach the same direction whenever either directional limit exists.

\paragraph{Initialization scale and training trajectories.}
\citet{woodworth2020kernel} show how initialization scale controls the transition between kernel and rich regimes in diagonal linear networks. \citet{moroshko2020initialization} study how initialization scale and training accuracy jointly determine implicit bias in diagonal linear classification under gradient flow. In the small-initialization limit, \citet{berthier2023incremental} characterizes successive coordinate activation in diagonal linear regression, while \citet{pesme2023saddle} describe saddle visits and the intervening motion through a time reparameterization. For matrix sensing with small initialization, \citet{jin2023incremental} show that GD successively approaches solutions of increasing rank. We study large initialization in linear logistic GD, obtaining a continuous geometric path and quantitative approximation for fixed-step updates through correction and subsequent margin growth.

\paragraph{Sequential training and path observables.}
Continual linear classification connects training geometry with forgetting. \citet{evron2023continual} study successive tasks trained to convergence on exponential loss with regularization toward the preceding parameter; the scaled task solutions approach sequential maximum-margin projections as regularization vanishes. \citet{jung2025continual} analyze unregularized GD with a fixed update budget per task, proving convergence toward the joint maximum-margin direction for cyclic or random sequences of jointly separable tasks, together with loss and forgetting bounds. Our analysis follows the trajectory within a task as its initialization scale grows, relating intermediate positions to evaluation-loss peaks and losses accumulated in GD time.

\section{Construction of the Ideal Path}
\label{sec:ideal-path}

We first specify logistic GD and its initialization, then construct the ideal path through negative-margin correction (Stage~1) and minimum-margin growth (Stage~2).

\subsection{Logistic GD at large initialization}
\label{sec:large-initialization}

Let $z_1,\ldots,z_n\in\mathbb R^d$ be fixed signed feature vectors, with each input multiplied by its label in $\{-1,+1\}$. Sample $i$ has margin $z_i^\top w$. For fixed weights $p_i>0$ with $\sum_i p_i=1$, the logistic loss is
\begin{equation*}
L(w)=\sum_{i=1}^n p_i\log\bigl(1+e^{-z_i^\top w}\bigr).
\end{equation*}
We assume strict separability: there is a vector $a$ such that $z_i^\top a>0$ for every $i$. Full-batch gradient descent with a fixed step size $\eta>0$ gives
\begin{equation*}
w_{k+1}=w_k+\eta\sum_{i=1}^n p_iq(z_i^\top w_k)z_i,
\qquad q(m)=\frac{1}{1+e^m}.
\end{equation*}
We study a family of initializations
\begin{equation}
w_0=Ru_0+e_R,\qquad R\longrightarrow\infty,\qquad
\frac{\|e_R\|}{R}\longrightarrow0,
\label{eq:initialization}
\end{equation}
where $u_0\in\mathbb R^d$ is fixed and may have negative or zero margins. The data, weights, and step size remain fixed. We observe $w_k/R$, so the initial positions approach $u_0$ while the paths retain changes in both norm and direction.

Such initializations arise after prolonged, unregularized GD on a previous separable linear logistic task. From a fixed initialization and with a sufficiently small fixed step size, the parameter $v_{T_{\mathrm{old}}}$ has diverging norm and a limiting Euclidean maximum-margin direction \citep[Theorem~3]{soudry2018implicit}. Taking $R=\|v_{T_{\mathrm{old}}}\|$ and $u_0$ as that limiting unit direction gives~\eqref{eq:initialization}; $\|e_R\|/R$ is the remaining directional error.

For a fixed negative margin $m$, $q(Rm)$ approaches one; for a fixed positive $m$, it is asymptotic to $e^{-Rm}$. Figure~\ref{fig:two-stage-path} follows the resulting two stages in a single example.

\subsection{Stage 1: Negative-Margin Correction}
\label{sec:correcting-margins}

Measure time by $\tau=\eta k/R$. Away from zero margins, the limiting velocity (the path derivative with respect to $\tau$) is the sum $\sum_{i:z_i^\top x<0}p_i z_i$. Every sample with a negative margin contributes its weighted feature vector. While the margin signs remain unchanged, this sum is constant and the path follows a straight line.

At a zero margin, the corresponding sample can contribute a fraction of its weight. The velocity is the vector of smallest Euclidean norm among
\begin{equation}
\sum_{i:z_i^\top x<0}p_i z_i
+\sum_{i:z_i^\top x=0}p_i\theta_i z_i,
\qquad 0\le\theta_i\le1.
\label{eq:correction-velocity}
\end{equation}
This rule determines whether the path leaves the boundary or travels along it. In Figure~\ref{fig:two-stage-path}(a), both samples initially contribute. When the second margin reaches zero, that sample ceases to contribute, and the path turns to correct the remaining error.

These rules have a compact description through the total negative margin,
\begin{equation}
H(x)=\sum_i p_i[-z_i^\top x]_+,\qquad
\dot y(\tau)\in-\partial H(y(\tau)),\qquad y(0)=u_0,
\label{eq:correction-flow}
\end{equation}
where $[t]_+=\max\{t,0\}$ and $\partial H$ is the convex subdifferential. This flow selects exactly the minimum-norm velocity in~\eqref{eq:correction-velocity} \citep[Proposition~2.2]{santambrogio2017gradient}. The value of $H$ decreases along the path until it reaches zero. Under strict separability, this happens in finite time. Write
\begin{equation*}
C=\{x:z_i^\top x\ge0\ \text{for all }i\},\qquad
\tau_*=\inf\{\tau:y(\tau)\in C\},\qquad E=y(\tau_*).
\end{equation*}
The correction flow stops at $E$, but GD continues; Stage~2 describes its subsequent motion on a different time scale.

\begin{figure}[htbp]
\centering
\includegraphics[width=\linewidth]{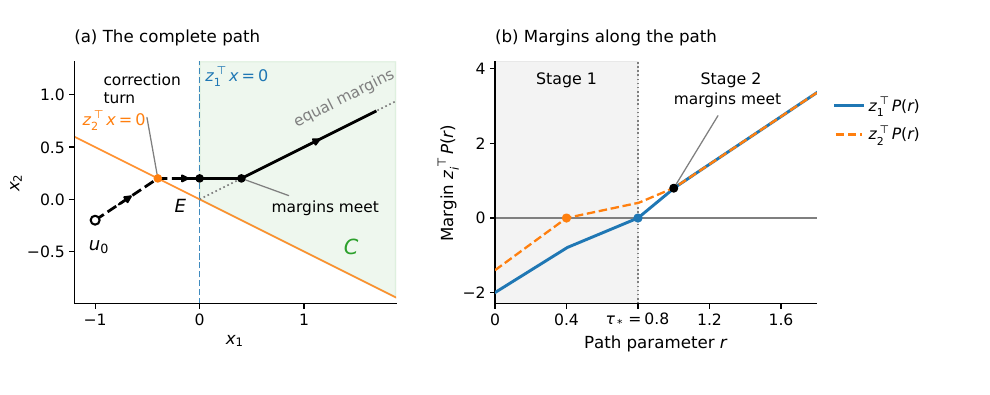}
\caption{\textbf{The two-stage ideal path.}
The analytic example uses $z_1=(2,0)$, $z_2=(1,2)$, $p_1=p_2=1/2$, and $u_0=(-1,-1/5)$.
(a) The dashed correction path turns when the second margin reaches zero and reaches the shaded nonnegative region $C$ at $E=(0,1/5)$. The solid continuation follows $z_1$ until the margins meet, then follows their equal-margin line with velocity $(8/5,4/5)$.
(b) The same path's margins, $z_i^\top P(r)$, show the correction turn at $r=2/5$, entry at $\tau_*=4/5$, and persistent equality from $r=1$. The horizontal axis uses the two-stage parameter defined in~\eqref{eq:complete-path}.}
\label{fig:two-stage-path}
\end{figure}

\subsection{Stage 2: Minimum-Margin Growth}
\label{sec:smallest-margins}

As positive margins grow, the gradient contributions become small, although their relative weights may change. Let $Q(w)=\sum_i p_iq(z_i^\top w)$. After entry into $C$, assign each GD step an increment $\eta Q(w_k)/R$ in the local time parameter $s$. The normalized displacement per unit $s$ is the weighted average $\sum_i p_iq(z_i^\top w_k)z_i/Q(w_k)$.

At a point with positive normalized margins, a fixed gap between two margins produces an exponentially large ratio between their contributions as $R$ grows. Thus a unique smallest-margin sample determines the limiting velocity. The path follows this feature vector until another sample attains the same minimum margin.

When several samples share the smallest margin, their relative contributions determine how the path continues. Define
\begin{equation*}
b(x)=\min_i z_i^\top x,\qquad I(x)=\{i:z_i^\top x=b(x)\}.
\end{equation*}
We call the samples in $I(x)$ active. The path velocity is the minimum-norm point in the convex hull of the active feature vectors:
\begin{equation}
v(x)=\operatorname*{arg\,min}_{v\in\operatorname{conv}\{z_i:i\in I(x)\}}\|v\|,
\qquad
\dot x_E(s)=v(x_E(s)),\qquad x_E(0)=E.
\label{eq:margin-flow}
\end{equation}
The equation holds almost everywhere. Equivalently, $x_E$ is the convex subgradient flow of $-b$, which gives a unique continuation from $E$.

The projection in~\eqref{eq:margin-flow} can select a single vector or a combination. Samples that remain tied gain margin at the same rate; those that gain margin faster than the minimum leave the active set. In the example, $z_1$ initially determines the motion from $E$. Once the two margins meet, the velocity becomes $3z_1/5+2z_2/5$, which keeps them equal. For a fixed entry point $E$, this geometric rule does not involve the sample weights. The weights still affect the complete path through the correction stage and its endpoint.

\subsection{Geometric Properties of the Ideal Path}
\label{sec:complete-path}

The two stages meet at the same position $E$. We combine them into the \emph{ideal path}
\begin{equation}
P(r)=
\begin{cases}
y(r),&0\le r\le\tau_*,\\
x_E(r-\tau_*),&r>\tau_*.
\end{cases}
\label{eq:complete-path}
\end{equation}
The first stage uses $\tau=r$ and the second uses the local time $s=r-\tau_*$. The corresponding increments in rescaled time are $\eta/R$ and $\eta Q(w_k)/R$, respectively. Both stages use the normalization $w_k/R$ and the same fixed-step GD updates; only the time parameter changes.

\Needspace{5\baselineskip}
\begin{proposition}[Geometry of the complete path]
\label{prop:complete-path}
Under the assumptions of Section~\ref{sec:large-initialization}, the correction stage reaches $C$ in finite time. The two flows determine a unique continuous path $P$ on $[0,\infty)$, with finitely many changes of velocity. Between changes, the path is a straight segment traversed at constant speed.
\end{proposition}

Within each stage, a change of velocity lowers the speed, and only finitely many velocities are possible. The path can therefore be constructed segment by segment using the two minimum-norm rules. It stays continuous at $E$, although its velocity may change.

If $u_0\in C$, then $\tau_*=0$ and $E=u_0$. The continuation is also well defined when $E=0$, so the same construction covers both cases.

\section{Convergence Results}
\label{sec:path-convergence}

\subsection{Uniform Convergence of GD Trajectories}
\label{sec:uniform-convergence}

We now compare the ideal path with the normalized GD trajectory under the two-stage time parameterization. For each initialization scale $R$, let
\begin{equation*}
\kappa_R=\min\{k\ge0:w_k\in C\},\qquad
\tau_R=\frac{\eta\kappa_R}{R}
\end{equation*}
be the index of the first GD iterate in $C$ and its rescaled iteration time. Lemma~\ref{lem:app-finite-entry} ensures that $\kappa_R$ is finite. We assign rescaled times to the iterates:
\begin{equation}
r_{R,k}=
\begin{cases}
\eta k/R,&k\le\kappa_R,\\[2pt]
\displaystyle\tau_R+\frac{\eta}{R}
\sum_{j=\kappa_R}^{k-1}Q(w_j),&k>\kappa_R,
\end{cases}
\qquad P_R(r_{R,k})=\frac{w_k}{R},
\label{eq:actual-path}
\end{equation}
and interpolate $P_R$ linearly between consecutive nodes. The node times increase without bound, so this defines a path for every $r\ge0$. Its time scale changes at the actual entry time $\tau_R$. The ideal path changes at $\tau_*$; the difference between these times is part of the approximation error.

\begin{theorem}[Convergence of the complete training path]
\label{thm:path-convergence}
Under the assumptions of Section~\ref{sec:large-initialization}, let $P_R$ be the path in~\eqref{eq:actual-path} and $P$ the ideal path in~\eqref{eq:complete-path}. For every fixed $T>0$, there is a constant $C_T$, depending only on the data, weights, $u_0$, $\eta$, and $T$, such that, for all sufficiently large $R$,
\begin{equation}
\sup_{0\le r\le T}\|P_R(r)-P(r)\|
\le C_T\left(\frac{\|e_R\|}{R}+\frac{1}{\sqrt R}\right).
\label{eq:whole-path-bound}
\end{equation}
If, in addition,
\begin{equation}
0<\eta\Lambda\le2,\qquad
\Lambda=\frac14\lambda_{\max}\left(\sum_i p_i z_i z_i^\top\right),
\label{eq:step-condition}
\end{equation}
then $R^{-1/2}$ in~\eqref{eq:whole-path-bound} can be replaced by $\log(2+R)/R$, with a possibly different constant $C_T$.
\end{theorem}

The first term measures the normalized initialization error; the second measures the discrepancy between finite-scale GD and its geometric limit. Thus all initializations satisfying~\eqref{eq:initialization} approach the same path, uniformly on each fixed window.

\subsection{Proof Sketch}
\label{sec:convergence-mechanism}

The proof first compares the correction trajectory with the flow of the total negative margin. The scaled logistic loss satisfies
\begin{equation}
0\le\frac{L(Rx)}{R}-H(x)\le\frac{\log2}{R}
\qquad\text{for every }x.
\label{eq:loss-hinge-comparison}
\end{equation}
The time steps $\eta/R$ vanish as $R$ grows, even though $\eta$ is fixed. Convexity and bounded update directions then give uniform tracking of $y$ on each fixed interval of scaled iteration time, with error $O(\|e_R\|/R+R^{-1/2})$.

Next, strict separability controls arrival at $C$. Before ideal entry, $H$ decreases at a rate bounded away from zero, which limits how early GD can arrive. At ideal arrival, the tracking estimate bounds GD's remaining negative margins; separability controls the additional time needed to remove them. Together, these comparisons give
\begin{equation}
\left\|\frac{w_{\kappa_R}}R-E\right\|
+|\tau_R-\tau_*|
=O\left(\frac{\|e_R\|}{R}+R^{-1/2}\right).
\label{eq:entry-control}
\end{equation}
Thus GD passes between stages at a nearby position and a nearby time.

For the continuation, we compare GD in its local post-entry time with the flow of a smooth convex approximation to $-b$. This comparison remains valid through the short transition when some margins may become negative again. Convexity prevents the flow from amplifying its initial-position error. The two ideal stages also have bounded speed, so the entry-time difference contributes at most $M|\tau_R-\tau_*|$, where $M=\max_i\|z_i\|$. The same speed bound covers the interval when only one path has changed stages. Combining these comparisons with~\eqref{eq:entry-control} yields~\eqref{eq:whole-path-bound}.

Under~\eqref{eq:step-condition}, a GD update does not increase the distance between two parameter vectors \citep{bauschke2010baillon}. Comparing along the finitely many correction segments then improves the correction and entry estimates to $O(\|e_R\|/R+\log(2+R)/R)$. The same argument gives the stronger bound in Theorem~\ref{thm:path-convergence}.

\section{Path Functionals}
\label{sec:path-functionals}

\subsection{Convergence of Path Functionals}
\label{sec:functional-convergence}

A path functional assigns a value to an entire training trajectory. We now use path convergence to study two such quantities: peak evaluation loss and cumulative training loss.

Fix $S>\tau_*$ and stop at $N_R=\min\{k:r_{R,k}\ge S\}$. Write $\Gamma_R=((w_k,\eta k))_{k=0}^{N_R}$ for the actual trajectory, retaining its parameter values and GD times. We measure $\mathcal F[\Gamma_R]$ and choose a scale $a_R>0$.

\begin{proposition}[Convergence of path functionals]
\label{prop:functional-transfer}
Under the assumptions of Theorem~\ref{thm:path-convergence}, let $\Phi$ be a functional on continuous paths over $[0,S]$, continuous at $P$ in the uniform norm. If $\mathcal F[\Gamma_R]/a_R$ can be approximated by $\Phi[P_R]$ with an error tending to zero, then
\begin{equation*}
\frac{\mathcal F[\Gamma_R]}{a_R}
=\Phi[P_R]+o(1)
\longrightarrow\Phi[P].
\end{equation*}
\end{proposition}

Applying this principle requires identifying $\Phi$ and verifying the scaled approximation. For cumulative quantities, this includes converting GD time to the path parameter.

\subsection{Peak Evaluation Loss}
\label{sec:maximum-functional}

We first study the maximum $\max_{0\le k\le N_R}f(w_k)$ of a scalar-valued function $f$ along the training trajectory.

\begin{corollary}
\label{cor:maximum-functional}
Under the assumptions of Theorem~\ref{thm:path-convergence}, suppose $f(Rx)/a_R\to f_\infty(x)$ locally uniformly, where $f_\infty$ is continuous. Then
\begin{equation*}
\frac1{a_R}\max_{0\le k\le N_R}f(w_k)
\longrightarrow
\max_{0\le r\le S}f_\infty(P(r)).
\end{equation*}
\end{corollary}

Here $\Phi[\gamma]=\max_{0\le r\le S}f_\infty(\gamma(r))$, whose continuity follows from that of $f_\infty$. For a concrete choice, take finitely many fixed signed evaluation features $a_j$ with weights $\rho_j>0$, $\sum_j\rho_j=1$, and define
\begin{equation}
L_{\mathrm{ev}}(w)=\sum_j\rho_j\log(1+e^{-a_j^\top w}),\qquad
H_{\mathrm{ev}}(x)=\sum_j\rho_j[-a_j^\top x]_+.
\label{eq:evaluation-loss}
\end{equation}
The bound in~\eqref{eq:loss-hinge-comparison} gives $L_{\mathrm{ev}}(Rx)/R\to H_{\mathrm{ev}}(x)$ uniformly. Since $H_{\mathrm{ev}}$ is convex, its maximum on each straight segment occurs at an endpoint. The limiting peak is therefore determined by finitely many path vertices, including the stage-transition point $E$ and both endpoints. Its Lipschitz continuity also makes the normalized peak inherit the path error order.

For Figure~\ref{fig:functional-example}, append a zero third coordinate to the training vectors of Figure~\ref{fig:two-stage-path} and set $u_{0,3}=1$. With evaluation feature $a=(-2/5,1,-2/25)$ and $S=9/5$, the path crosses $a^\top x<0$ and then leaves. The normalized evaluation loss peaks at $1/25$ when $r=1$, while both endpoint margins are positive. Thus the actual peak is $R/25+o(R)$ even though the initial and final evaluation losses tend to zero.

\begin{figure}[htbp]
\centering
\includegraphics[width=\linewidth]{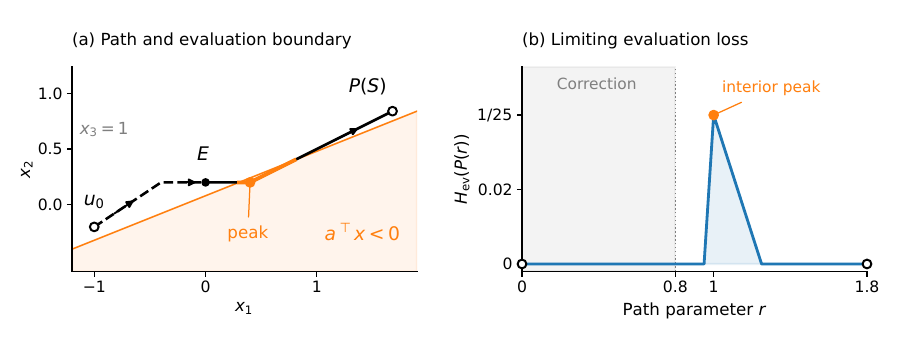}
\caption{\textbf{An intermediate peak in evaluation loss.}
The analytic example follows the training geometry of Figure~\ref{fig:two-stage-path} in the invariant plane $x_3=1$.
(a) The path crosses the shaded evaluation region $a^\top x<0$; correction is dashed and negative evaluation margins are marked in orange.
(b) The limiting normalized evaluation loss peaks at $1/25$, although it vanishes at both endpoints. Both panels are computed from the ideal path.}
\label{fig:functional-example}
\end{figure}

\subsection{Cumulative Training Loss}
\label{sec:cumulative-functional}

Unlike a maximum, cumulative loss depends on the GD time spent along the path. We sum the training loss in GD time $t=\eta k$ and split at the first iterate in $C$:
\begin{equation}
C_R^-=\eta\sum_{k=0}^{\kappa_R-1}L(w_k),\qquad
C_R^+=\eta\sum_{k=\kappa_R}^{N_R-1}L(w_k),\qquad
C_R=C_R^-+C_R^+.
\label{eq:cumulative-loss-definition}
\end{equation}
Since $S>\tau_*$, we have $N_R>\kappa_R$ for sufficiently large $R$. Writing $\Delta r_{R,k}=r_{R,k+1}-r_{R,k}$, the time change gives
\begin{equation*}
\frac{C_R^-}{R^2}=\sum_{k<\kappa_R}\Delta r_{R,k}\frac{L(w_k)}R,
\qquad
\frac{C_R^+}{R}=\sum_{k=\kappa_R}^{N_R-1}\Delta r_{R,k}\frac{L(w_k)}{Q(w_k)}.
\end{equation*}
The factor $1/Q$ accounts for the many updates needed to traverse the second stage. It compensates for the small training loss there.

\Needspace{10\baselineskip}
\begin{corollary}
\label{cor:cumulative-loss}
Under the assumptions of Theorem~\ref{thm:path-convergence}, for every fixed $S>\tau_*$,
\begin{equation}
\frac{C_R^-}{R^2}\longrightarrow\int_0^{\tau_*}H(y(r))\,dr,
\qquad
\frac{C_R^+}{R}\longrightarrow S-\tau_*.
\label{eq:cumulative-loss-limits}
\end{equation}
Here $H$ is the total negative margin defined in~\eqref{eq:correction-flow}.
In particular, $C_R/R^2\to\Phi[P]$, where the continuous path functional is
$\Phi[\gamma]=\int_0^S H(\gamma(r))\,dr$.
\end{corollary}

During correction, $L(Rx)/R\to H(x)$ turns the first sum into the integral. After entry, $L(w_k)$ can be replaced by $Q(w_k)$ with bounded total error in GD time, giving the second limit. The two stages thus contribute at different scales: after correction, small losses accumulate over many more updates.

\Needspace{9\baselineskip}
\section{Numerical Experiments}
\label{sec:experiments}

We test the trajectory and loss predictions using fixed-step, full-batch logistic GD initialized at $Ru_0$, constructing each ideal path independently from the data and $u_0$. Comparisons use the two-stage parameter in~\eqref{eq:actual-path} and an interval $[0,S]$ fixed across values of $R$. Path errors measure full parameter vectors. Maxima include every GD iterate through the stopping iterate; cumulative losses include the preceding updates.

\subsection{Trajectory Approximation}
\label{sec:experiments-paths}

We first use four equally weighted features spanning $\mathbb R^3$, with $\eta=1/2$ and a path that turns in both stages. This example tests approximation across segment changes and the transition between stages. Figure~\ref{fig:numerical-paths}(a) shows GD approaching the ideal geometry as $R$ grows. From $R=5$ to $R=80$, the uniform path error decreases from $0.0638$ to $0.0102$. This error measures agreement throughout the window, including the transitions. The analytic geometry of Figure~\ref{fig:two-stage-path}, tested on a shorter window, shows the same trend in Figure~\ref{fig:numerical-paths}(b).

We then vary the geometry beyond these two controlled examples. Across 40 random strictly separable problems with five samples in three dimensions, the median path error falls from $0.185$ at $R=5$ to $0.0370$ at $R=80$. Figure~\ref{fig:numerical-paths}(c) includes every instance. All comparisons retain the difference between actual and ideal entry times.

\begin{figure}[htbp]
\centering
\includegraphics[width=\linewidth]{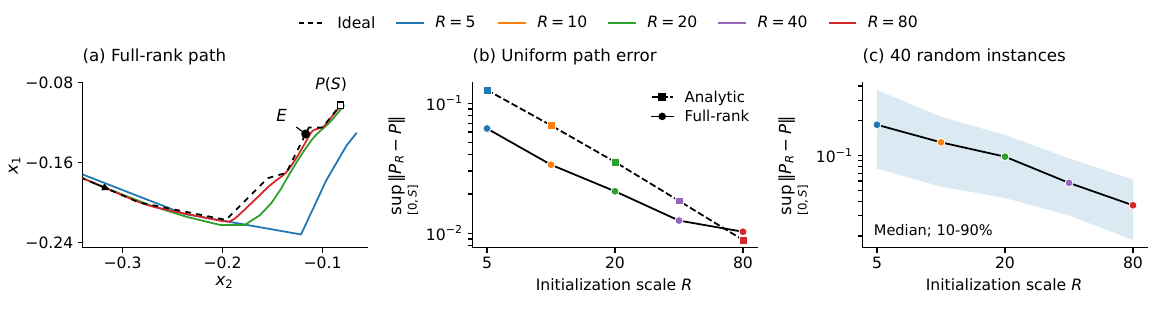}
\caption{\textbf{GD approaches the complete ideal path.}
(a) A coordinate projection of the four-sample example, omitting most of the initial approach to highlight the turns. Colored curves connect saved GD iterates divided by $R$; the dashed line shows the ideal path, with its entry $E$ and endpoint $P(S)$ marked.
(b) Uniform path errors for the analytic and four-sample examples.
(c) The median and 10th--90th percentile band over all 40 random instances. Errors in (b,c) use full parameter vectors over each problem's complete window $[0,S]$.}
\label{fig:numerical-paths}
\end{figure}

\Needspace{10\baselineskip}
\subsection{Peak and Cumulative Loss Predictions}
\label{sec:experiments-functionals}

We next test whether intermediate positions on the ideal path predict a temporary rise in evaluation loss. For the four-sample example, an evaluation feature excluded from the updates yields the interior peak in Figure~\ref{fig:numerical-functionals}(a). The peak occurs during correction, when the path enters and leaves the region of negative evaluation margin. The ideal path predicts a normalized peak of $0.09314$, compared with $0.09643$ at $R=80$. The corresponding unnormalized peak is $7.71$, while both endpoint losses are below $0.002$.

Cumulative loss additionally tests the conversion back to GD time. We split the sums at the first iterate in $C$ and include the step-size factor as in~\eqref{eq:cumulative-loss-definition}. For the same runs, the predictions are $C_R^-/R^2\to0.38277$ and $C_R^+/R\to0.10131$. At $R=80$, the observed values are $0.38876$ and $0.10421$. Figure~\ref{fig:numerical-functionals}(b) divides each cumulative loss by its predicted leading term. Ratios near one indicate agreement at each scale: small losses after entry are balanced by the many updates needed to traverse that part of the path.

\subsection{Fixed Features from Handwritten Digits}
\label{sec:experiments-digits}

To test the predictions in a higher-dimensional geometry, we use the 64 pixel features of the handwritten-digits dataset, divided by 16 and augmented with a bias coordinate \citep{alpaydin1998digits}. We train on 32 images of digits 5 and 6 with uniform weights and $\eta=1/\Lambda$. The initial direction is the unit maximum-margin separator of the 357 images of digits 3 and 8, also used for evaluation. This gives an initialization associated with a different classification task. Both problems are strictly separable.

We hold the data and initial direction fixed and vary only $R$. The runs at all five values $R\in\{5,10,20,40,80\}$ cover both stages on the fixed window. The uniform path error decreases from $0.1183$ to $0.01956$ in Figure~\ref{fig:numerical-functionals}(c). At $R=80$, the normalized evaluation peak is $0.035725$, close to the ideal prediction $0.035607$.

\begin{figure}[htbp]
\centering
\includegraphics[width=\linewidth]{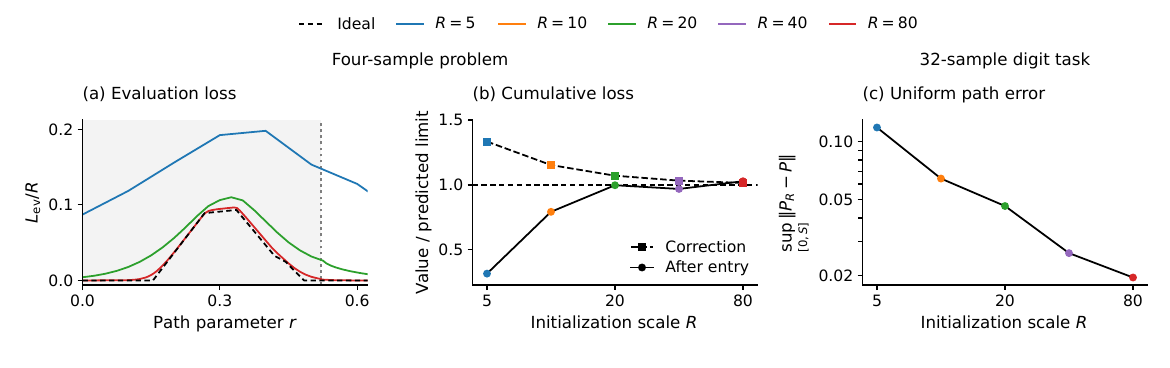}
\caption{\textbf{Functional predictions and a test on image features.}
Panels (a,b) use the four-sample example.
(a) Evaluation loss divided by $R$, compared with $H_{\mathrm{ev}}(P(r))$. Curves connect saved GD iterates; shading marks ideal correction.
(b) Cumulative training losses divided by $0.38277R^2$ before entry and $0.10131R$ after entry; the dashed line at one is the limit.
(c) Uniform path error for the 32-image digit task over its fixed window.}
\label{fig:numerical-functionals}
\end{figure}

\FloatBarrier
\Needspace{8\baselineskip}
\section{Conclusion}
\label{sec:conclusion}

The ideal path gives a finite geometric description of training at large initialization. For separable logistic GD, it connects negative-margin correction with subsequent minimum-margin growth. Our convergence bounds relate this construction to actual fixed-step updates on every fixed window of the two-stage parameter. The geometry and its relation to GD time yield leading-order limits for evaluation-loss peaks and cumulative training loss. These limits can thus be computed from the weighted data and limiting initial position.

\bibliography{references}
\bibliographystyle{informs2014}

\begin{APPENDICES}
\renewcommand*{\theHsection}{appendix.\Alph{section}}

\section{Geometry of the Ideal Path}
\label{app:ideal-geometry}

This section proves Proposition~\ref{prop:complete-path} and gives a finite procedure for constructing the ideal path.

We use the data and assumptions of Section~\ref{sec:large-initialization}. Choose a unit vector $a$ and $\gamma>0$ such that $a^\top z_i\ge\gamma$ for every $i$, and write $M=\max_i\|z_i\|$ and $p_*=\min_i p_i>0$.

\subsection{Negative-margin correction}
\label{app:correction-geometry}

The correction velocity is determined by the signs of the margins. Write
\[
N(x)=\{i:z_i^\top x<0\},\qquad Z(x)=\{i:z_i^\top x=0\}.
\]
Since $H(x)=\sum_i p_i[-z_i^\top x]_+$ is a finite convex function, its subdifferential is obtained by adding those of its summands. Thus
\begin{equation}
A(x):=-\partial H(x)
=\left\{\sum_{i\in N(x)}p_i z_i+
\sum_{i\in Z(x)}p_i\theta_i z_i:0\le\theta_i\le1\right\}.
\label{eq:app-correction-set}
\end{equation}
This is a nonempty compact convex set. Its point $v$ of smallest norm is unique, although the coefficients representing $v$ need not be unique.

\paragraph{Velocity selection at zero margins.}
For any minimizing coefficients in~\eqref{eq:app-correction-set}, optimality with respect to $\theta_i$ gives
\begin{equation}
z_i^\top v>0\ \Longrightarrow\ \theta_i=0,
\qquad
z_i^\top v<0\ \Longrightarrow\ \theta_i=1,
\qquad i\in Z(x).
\label{eq:app-zero-margin-rule}
\end{equation}
Indeed, the derivative of $\|v\|^2/2$ with respect to $\theta_i$ is $p_i z_i^\top v$. A positive derivative requires the lower endpoint of the interval $[0,1]$, and a negative derivative requires the upper endpoint.

Consider the line $x+tv$ for sufficiently small $t>0$. Margins that were nonzero retain their signs. A zero margin with positive derivative becomes positive, so its coefficient zero remains admissible; one with negative derivative becomes negative, so its coefficient one remains admissible. A zero margin with zero derivative stays on the boundary. Consequently,
\[
v\in A(x+tv)\subseteq A(x).
\]
The same vector therefore remains the minimum-norm point and the line solves $\dot y\in-\partial H(y)$. Section~\ref{app:finite-construction} shows that these local segments extend to a unique path for all times.

\subsection{Minimum-margin growth}
\label{app:margin-geometry}

The second stage follows the samples attaining the smallest margin. Since $-b(x)=\max_i(-z_i^\top x)$ is convex,
\[
-\partial(-b)(x)=\operatorname{conv}\{z_i:i\in I(x)\}.
\]
Let $v=v(x)$ be the minimum-norm point of this convex hull. Projection optimality gives
\begin{equation}
z_i^\top v\ge\|v\|^2\qquad(i\in I(x)).
\label{eq:app-margin-projection}
\end{equation}
Write $v=\sum_{i\in I(x)}\alpha_i z_i$, where $\alpha_i\ge0$ and $\sum_i\alpha_i=1$. The $\alpha_i$-weighted average of the left-hand side of~\eqref{eq:app-margin-projection} equals $\|v\|^2$. Every index with positive coefficient must therefore attain equality. In particular,
\[
v\in\operatorname{conv}\{z_i:i\in J\},\qquad
J=\{i\in I(x):z_i^\top v=\|v\|^2\}.
\]

Along $x+tv$, the samples in $J$ remain tied. The other samples in $I(x)$ gain margin faster and immediately leave the minimum. Initially inactive samples have a positive margin gap and cannot reach the minimum for sufficiently small $t>0$. Thus $J$ is the active set on this short segment, and the convex hull of its feature vectors still has $v$ as its minimum-norm point. The line solves the required subgradient equation. Repeating this construction gives existence directly, including when the initial point lies on the boundary of $C$; the finite extension is proved below.

All convex combinations of the feature vectors have inner product at least $\gamma$ with $a$. Hence $\gamma\le\|v(x)\|\le M$, and along the constructed path,
\begin{equation}
\frac{d}{ds}b(x_E(s))=\|v(x_E(s))\|^2\ge\gamma^2
\quad\text{for almost every }s.
\label{eq:app-minimum-margin-growth}
\end{equation}
In particular, $b(x_E(s))\ge b(E)+\gamma^2s$. Starting from any $E\in C$, the ideal continuation has strictly positive margins at every $s>0$.

\subsection{Proof of Proposition~\ref{prop:complete-path}}
\label{app:finite-construction}

\begin{proof}
The local arguments above show that the velocity persists until a new margin equality is reached. At such an event, several samples may arrive together; all of them enter the next projection problem.

\paragraph{Finitely many events.}
Let $v_-$ and $v_+$ denote the velocities immediately before and after an event within either stage. The old velocity remains feasible at the event point. For correction, a margin arriving at zero replaces its previously fixed coefficient by an interval containing that coefficient. For minimum-margin motion, the previously tied samples remain among those attaining the minimum. In both cases, the new velocity is the projection of zero onto a convex set containing $v_-$. Therefore
\begin{equation}
v_+^\top(v_--v_+)\ge0,
\qquad
\|v_-\|^2\ge\|v_+\|^2+\|v_--v_+\|^2.
\label{eq:app-velocity-decrease}
\end{equation}
Every actual change of velocity strictly decreases its norm.

There are only finitely many possible velocities: each sign pattern determines one correction projection, and each nonempty active subset determines one minimum-margin projection. Equation~\eqref{eq:app-velocity-decrease} thus permits only finitely many velocity changes in each stage. Events that leave the velocity unchanged cannot accumulate either. On a straight segment, each correction margin is affine; in the second stage, each gap $z_i^\top x-b(x)$ is affine by~\eqref{eq:app-minimum-margin-growth}. A nonconstant affine function reaches zero at most once, while a persistent zero requires no new event. Thus there are only finitely many events between velocity changes.

These observations extend the local solutions globally and exclude an accumulation of event times. The second stage eventually continues along a straight ray when no further margin can reach the minimum.

\paragraph{Finite arrival at the nonnegative region.}
Along each segment, the zero-margin rule also gives
\begin{equation}
\frac{d}{d\tau}H(y(\tau))=-\|\dot y(\tau)\|^2.
\label{eq:app-correction-energy}
\end{equation}
On the interior of each segment, the zero margins have zero derivative, and the derivative of $H$ is the negative weighted sum of the negative-margin derivatives. Taking the inner product of~\eqref{eq:app-correction-set} with the velocity gives~\eqref{eq:app-correction-energy}.

Outside $C$, at least one sample contributes its full weight. All other coefficients are nonnegative, so every $v\in A(x)$ satisfies
\[
a^\top v\ge\gamma\sum_{i\in N(x)}p_i\ge\gamma p_*.
\]
The selected velocity thus satisfies $\gamma p_*\le\|v\|\le M$. If the path remained outside $C$ for longer than $H(u_0)/(\gamma p_*)^2$, equation~\eqref{eq:app-correction-energy} would force $H$ below zero. Hence the arrival time is finite, with
\begin{equation}
\tau_*\le\frac{H(u_0)}{(\gamma p_*)^2},
\qquad
H(y(\tau))\ge(\gamma p_*)^2(\tau_*-\tau)
\quad(0\le\tau\le\tau_*).
\label{eq:app-correction-arrival}
\end{equation}
The second inequality follows by integrating~\eqref{eq:app-correction-energy} from $\tau$ to $\tau_*$. At any point in $C$, the set $A(x)$ contains zero, so the correction velocity is zero. The correction flow therefore remains at its arrival position $E$.
\par

\paragraph{Uniqueness.}
For either convex function $f=H$ or $f=-b$, let $x$ and $\widetilde x$ be two solutions, with $g=-\dot x\in\partial f(x)$ and $\widetilde g=-\dot{\widetilde x}\in\partial f(\widetilde x)$. The two subgradient inequalities imply
\[
(x-\widetilde x)^\top(g-\widetilde g)\ge0,
\qquad
\frac12\frac{d}{dt}\|x-\widetilde x\|^2\le0
\quad\text{almost everywhere}.
\]
Solutions with the same initial point coincide. More generally, their distance never exceeds its initial value. This proves uniqueness for both of the constructed flows.

\paragraph{Connecting the stages.}
The correction stage ends at $E$ and the second stage starts at $E$. Equation~\eqref{eq:complete-path} therefore gives a unique continuous path, with speed at most $M$ and finitely many changes of velocity. Each stage uses its own parameter, so~\eqref{eq:app-velocity-decrease} is a statement within a stage; it imposes no speed ordering across the connection.

If $u_0\in C$, then $\tau_*=0$ and the path starts directly with the second stage. If $E=0$, let $v_*$ be the minimum-norm point of $\operatorname{conv}\{z_1,\ldots,z_n\}$. Projection optimality gives $z_i^\top v_*\ge\|v_*\|^2$ for every $i$, and $v_*$ belongs to the convex hull of the equality indices. Hence $x_E(s)=sv_*$ solves the second-stage equation for all $s\ge0$. Strict separability gives $v_*\ne0$, so entry at $E=0$ also yields a nontrivial continuation.
\end{proof}

\subsection{Finite Construction of the Ideal Path}
\label{app:ideal-path-construction}

\paragraph{Computing the next event.}
The construction can be implemented using the two projection problems and positive arrival times. At a current position $x$ with selected velocity $v$, the next correction event occurs after
\begin{equation}
\Delta^-=\min_{i:(z_i^\top x)(z_i^\top v)<0}
\left(-\frac{z_i^\top x}{z_i^\top v}\right).
\label{eq:app-correction-event}
\end{equation}
Only currently nonzero margins enter this minimum. Zero margins either remain zero or leave immediately according to~\eqref{eq:app-zero-margin-rule}. Outside $C$, the candidate set is nonempty: otherwise the same velocity would persist forever and force $H$ below zero.

In the second stage, the next event occurs after
\begin{equation}
\Delta^+=\min_{\substack{i\notin I(x)\\z_i^\top v<\|v\|^2}}
\frac{z_i^\top x-b(x)}{\|v\|^2-z_i^\top v}.
\label{eq:app-margin-event}
\end{equation}
An empty minimum means $\Delta^+=\infty$ and the current line continues forever. Initially tied samples that leave the minimum have larger margin derivatives and cannot return while this velocity remains unchanged.

\Needspace{9\baselineskip}
\noindent\textbf{Constructing the complete ideal path.}
\begin{enumerate}
\item Start at $x=u_0$ and $r=0$ in the correction stage.
\item If the current stage is correction and $x\in C$, record $E=x$ and $\tau_*=r$, and switch to minimum-margin growth.
\item Compute the minimum-norm velocity from~\eqref{eq:app-correction-set} or~\eqref{eq:margin-flow}, and the corresponding duration from~\eqref{eq:app-correction-event} or~\eqref{eq:app-margin-event}.
\item Append $P(r+t)=x+tv$ for $0\le t\le\Delta$. If $\Delta=\infty$, finish with this ray. Otherwise set $x\leftarrow x+\Delta v$, $r\leftarrow r+\Delta$, and return to step 2, including every simultaneous arrival in the next projection.
\end{enumerate}
The procedure uses only the weighted data and $u_0$. Persistent equalities are handled by the projection itself, so no zero-duration steps are needed.

\section{Proof of Path Convergence}
\label{app:path-convergence}

We prove Theorem~\ref{thm:path-convergence} by keeping track of three quantities: the correction-path error, the entry-position error, and the difference between the two entry times. The estimates after entry are uniform over the actual nonnegative starting point. This allows the same argument to cover initialization perturbations and the case $E=0$.

\subsection{Basic estimates and path parameterization}
\label{app:basic-estimates}

Use $M,p_*,a,\gamma$ from Appendix~\ref{app:ideal-geometry}, and set
\[
g(w)=\sum_i p_iq(z_i^\top w)z_i=-\nabla L(w),\qquad
\delta_R=\|e_R\|/R.
\]
Constants below may depend on the fixed data, weights, $u_0$, $\eta$, and the indicated fixed time window, but not on $R$ or the direction of $e_R$. We work with $R\ge1$ and $\delta_R\le1$, as holds eventually under~\eqref{eq:initialization}.

For $\ell(m)=\log(1+e^{-m})$, direct scalar inequalities give
\begin{equation*}
0\le\ell(m)-[-m]_+\le\log2,
\qquad 0<q(m)\le\ell(m).
\end{equation*}
Consequently,
\begin{equation}
\gamma Q(w)\le a^\top g(w)\le\|g(w)\|\le M Q(w)\le M L(w),
\qquad Q(w)\le1.
\label{eq:app-gradient-basic}
\end{equation}
The Hessian also satisfies
\begin{equation}
0\preceq\nabla^2L(w)\preceq\frac14\sum_i p_i z_i z_i^\top\preceq\Lambda I,
\qquad \|\nabla^2L(w)\|\le M^2 L(w).
\label{eq:app-hessian-bounds}
\end{equation}
The first bound uses $q(1-q)\le1/4$; the second uses $q(1-q)\le\ell$. Since $|\ell'(m)|/\ell(m)\le1$, for every displacement $d$,
\begin{equation}
L(w+d)\le e^{M\|d\|}L(w).
\label{eq:app-loss-displacement}
\end{equation}

Lemma~\ref{lem:app-finite-entry} below proves $\kappa_R<\infty$. The nodes in~\eqref{eq:actual-path} then have positive increments at most $\eta/R$, and their interpolated path has speed at most $M$ in both stages. They cover the whole nonnegative parameter axis. Indeed, if $\sum_{k\ge\kappa_R}\eta Q(w_k)$ were finite, the sum of the update norms would be finite by~\eqref{eq:app-gradient-basic}. The iterates would converge to a finite vector $w$, where continuity would give $Q(w_k)\to Q(w)>0$, a contradiction.

\subsection{Correction-Stage Approximation}
\label{app:correction-tracking}

Let $Y_R$ be the normalized, linearly interpolated GD trajectory on the iteration-time grid $\tau_k=\eta k/R$, without stopping at entry.

\begin{lemma}[Correction-stage approximation]
\label{lem:app-correction-approximation}
For every fixed $T_0$,
\begin{equation}
\sup_{0\le\tau\le T_0}\|Y_R(\tau)-y(\tau)\|
\le\left(\delta_R^2+\frac{(2\log2+4\eta M^2)T_0}{R}\right)^{1/2}
\le\delta_R+\frac{C_{T_0}}{\sqrt R}.
\label{eq:app-correction-bound}
\end{equation}
\par
No step-size restriction is used here.
\end{lemma}

\begin{proof}[Proof]
Define $F_R(x)=L(Rx)/R$. Then
\[
0\le F_R(x)-H(x)\le\frac{\log2}{R},\qquad
\dot Y_R(\tau)=-\nabla F_R(Y_R(\tau_k))
\quad(\tau_k<\tau<\tau_{k+1}).
\]
Both $F_R$ and $H$ are $M$-Lipschitz, and $\|Y_R(\tau)-Y_R(\tau_k)\|\le\eta M/R$. Convexity, applied at $Y_R(\tau_k)$ for $F_R$ and at $y(\tau)$ for $H$, yields
\[
\begin{split}
\frac12\frac{d}{d\tau}\|Y_R-y\|^2
&\le F_R(y)-F_R(Y_R(\tau_k))+H(Y_R)-H(y)
   +M\|Y_R-Y_R(\tau_k)\|\\
&\le\frac{\log2+2\eta M^2}{R}.
\end{split}
\]
Integrating from $\|Y_R(0)-y(0)\|=\delta_R$ gives~\eqref{eq:app-correction-bound}.
\end{proof}

\subsection{Bounds on Entry Position and Time}
\label{app:entry-control}

A path estimate bounds the negative margins remaining near the ideal arrival time. The following bound uses their magnitude to bound the additional GD time.

\begin{lemma}[Finite entry from bounded negative margins]
\label{lem:app-finite-entry}
Suppose GD starts at $\widetilde w_0$ with $z_i^\top\widetilde w_0\ge-B$ for every $i$, where $B\ge0$. Let $\kappa$ be its first node in $C$. Put
\[
c_\eta=\log(2/p_*)+\eta M^2/2+1,
\qquad A_B=(2B+c_\eta)/\gamma.
\]
Then $\kappa<\infty$, and
\begin{equation}
\eta\sum_{k<\kappa}Q(\widetilde w_k)\le\frac{A_B^2}{2(B+1)},
\qquad
\eta\kappa\le\frac{A_B^2}{p_*(B+1)},
\qquad
\sum_{k<\kappa}\|\widetilde w_{k+1}-\widetilde w_k\|
\le\frac{MA_B^2}{2(B+1)}.
\label{eq:app-raw-entry-bound}
\end{equation}
In particular, all three quantities are bounded by a constant times $B+1$.
\end{lemma}

\begin{proof}
Use the comparison point $\bar w=\widetilde w_0+A_Ba$ and set $D_k=\|\widetilde w_k-\bar w\|^2$. Abbreviate $Q_k=Q(\widetilde w_k)$ and $g_k=g(\widetilde w_k)$. Writing a margin as $m=\log((1-q(m))/q(m))\le\log(1/q(m))$ and using the concavity of $q\log(1/q)$ gives
\[
\sum_i p_iq(z_i^\top\widetilde w_k)z_i^\top\widetilde w_k
\le Q_k\log(1/Q_k).
\]
The lower bounds on the initial margins and on $z_i^\top a$ imply
\[
(\widetilde w_k-\bar w)^\top g_k
\le Q_k\{\log(1/Q_k)+B-A_B\gamma\}.
\]
Expanding the squared distance after one update therefore yields
\[
D_{k+1}-D_k
\le2\eta Q_k\{\log(1/Q_k)+B-A_B\gamma+\eta M^2Q_k/2\}.
\]
Before entry, a negative margin contributes more than half its weight, so $Q_k>p_*/2$. The expression in braces is at most $-B-1$. Thus
\[
D_{k+1}-D_k\le-2\eta Q_k(B+1)\qquad(k<\kappa).
\]
If entry never occurred, the right-hand side would make $D_k$ negative. Summing up to the entry update and using $D_0=A_B^2$ proves the first bound in~\eqref{eq:app-raw-entry-bound}. The lower bound on $Q_k$ gives the time bound, and $\|\widetilde w_{k+1}-\widetilde w_k\|\le\eta M Q_k$ gives the length bound. If $\widetilde w_0\in C$, all three sums are zero.
\end{proof}

\begin{lemma}[Entry stability]
\label{lem:app-entry-stability}
Take $T_0=\tau_*+\eta+1$, and suppose that the left-hand side of~\eqref{eq:app-correction-bound} is at most $\varepsilon$.
Let $Y_R^{\mathrm{stop}}$ agree with $Y_R$ until $\tau_R$ and stay at $w_{\kappa_R}/R$ thereafter.
Then
\begin{equation}
\sup_{\tau\ge0}\|Y_R^{\mathrm{stop}}(\tau)-y(\tau)\|
+\left\|\frac{w_{\kappa_R}}R-E\right\|
+|\tau_R-\tau_*|
\le C\left(\varepsilon+\frac1R\right).
\label{eq:app-entry-joint}
\end{equation}
\par
\end{lemma}

\begin{proof}[Proof]
At $j_* =\lceil R\tau_*/\eta\rceil$, the ideal correction flow has reached $E$, so every unnormalized GD margin is at least $-MR\varepsilon$. Lemma~\ref{lem:app-finite-entry}, applied from this node, gives
\[
(\tau_R-\tau_*)_+\le\frac\eta R+C\left(\varepsilon+\frac1R\right).
\]
This bounds the first entry time even if GD had already entered $C$ earlier. Conversely, if $\tau_R<\tau_*$, then $H(Y_R(\tau_R))=0$. Since $H$ is $M$-Lipschitz, equation~\eqref{eq:app-correction-arrival} gives
\[
(\gamma p_*)^2(\tau_*-\tau_R)\le H(y(\tau_R))\le M\varepsilon.
\]
The latter case is absent when $\tau_*=0$.

Before the earlier entry, use the path estimate. If GD enters early, use the speed bound for $y$ over the intervening interval. If GD enters late, its displacement after $j_*$ is controlled by the length bound in Lemma~\ref{lem:app-finite-entry}; the interval from $\tau_*$ to $\tau_{j_*}$ adds at most $\eta M/R$. These comparisons also control the two final positions. Thus~\eqref{eq:app-entry-joint} follows.
This deduction does not assume in advance that the actual entry lies in the comparison window.
\end{proof}
Combining Lemmas~\ref{lem:app-correction-approximation} and~\ref{lem:app-entry-stability} proves the basic entry bounds, including their dependence on $e_R$.

\subsection{Post-Entry Trajectory Approximation}
\label{app:post-entry}

We first compare GD with a smooth reference flow from any unnormalized initial point $w^{\mathrm{in}}\in C$. The constants in this comparison are independent of the norm and individual margins of $w^{\mathrm{in}}$.

\begin{lemma}[Uniform bound on iterations to reach small loss]
\label{lem:app-small-loss}
Set
\[
\delta=\min\{p_*/2,(4\eta M^2)^{-1}\},\qquad
D_\eta=2\eta/e+\eta^2M^2,
\qquad N_0=\left\lceil\frac{4D_\eta}{\eta^2\gamma^2\delta^2}\right\rceil+1.
\]
From every $w^{\mathrm{in}}\in C$, there is a node $j<N_0$ with $L(w_j)\le\delta$. For all $k\ge j$,
\begin{equation}
L(w_{k+1})\le L(w_k)-\frac\eta2\|g(w_k)\|^2,
\qquad L(w_k)\le\delta,
\label{eq:app-small-loss-descent}
\end{equation}
all margins are strictly positive, and
\begin{equation}
\eta\sum_{k\ge j}Q(w_k)^2\le\frac{2\delta}{\gamma^2}.
\label{eq:app-q-square-sum}
\end{equation}
\end{lemma}

\begin{proof}
Write $h_k=w_k-w^{\mathrm{in}}$. The nonnegative initial margins give
\[
h_k^\top g(w_k)
\le\sum_i p_i(z_i^\top w_k)q(z_i^\top w_k)\le1/e,
\]
using $mq(m)\le me^{-m}\le1/e$ for $m\ge0$ and $mq(m)<0$ for $m<0$. Hence $\|h_{N_0}\|^2\le N_0D_\eta$, while strict separability gives
\[
\eta\gamma\sum_{k<N_0}Q(w_k)\le a^\top h_{N_0}\le\sqrt{N_0D_\eta}.
\]
Some $j<N_0$ therefore has $Q(w_j)<\delta/2$. At this node every $q(z_i^\top w_j)\le Q(w_j)/p_*\le1/4$, so $L(w_j)\le2Q(w_j)<\delta$.

If $L(w_k)\le\delta$, equations~\eqref{eq:app-hessian-bounds} and~\eqref{eq:app-loss-displacement} bound the Hessian on the update segment by $M^2e^{1/4}L(w_k)$. Taylor's theorem then gives
\[
L(w_{k+1})\le L(w_k)-\eta\left(1-\frac{\eta M^2e^{1/4}L(w_k)}2\right)\|g(w_k)\|^2
\le L(w_k)-\frac\eta2\|g(w_k)\|^2.
\]
Thus this loss region is invariant. Summing the descent inequality and using $\|g(w_k)\|\ge\gamma Q(w_k)$ proves~\eqref{eq:app-q-square-sum}. Each sample loss is at most $\delta/p_*\le1/2<\log2$, which implies strictly positive margins.
\end{proof}

Define the reference field and flow by
\begin{equation*}
F(w)=\frac{\sum_i p_i e^{-z_i^\top w}z_i}{\sum_i p_i e^{-z_i^\top w}},
\qquad \dot U(t)=F(U(t)),\qquad U(0)=w^{\mathrm{in}}.
\end{equation*}

Parameterize the GD interpolation $W$ from $w^{\mathrm{in}}$ by the time nodes $t_k=\eta\sum_{l<k}Q(w_l)$, before division by $R$.

\begin{lemma}[Uniform GD--reference-flow approximation]
\label{lem:app-reference-comparison}
For every $w^{\mathrm{in}}\in C$,
\begin{equation}
\sup_{t\ge0}\|W(t)-U(t;w^{\mathrm{in}})\|\le C_0,
\label{eq:app-shadow-bound}
\end{equation}
\par
with $C_0$ independent of $w^{\mathrm{in}}$.
\end{lemma}

\begin{proof}[Proof]
The field is the negative gradient of the convex function $\log\sum_i p_i e^{-z_i^\top w}$. Its derivative is the negative covariance matrix of the feature vectors under the normalized exponential weights. Therefore $\|F\|\le M$, $\operatorname{Lip}(F)\le M^2$, and
\[
(u-v)^\top(F(u)-F(v))\le0.
\]
In particular, the flow exists for all times and does not increase initial distances.

On an interval between nodes, $\dot W=g(w_k)/Q(w_k)$. Choose $j<N_0$ as in Lemma~\ref{lem:app-small-loss}. In the small-loss region, $q_i\le1/2$ and $e^{-z_i^\top w_k}=q_i/(1-q_i)$. Since
\[
\sum_i p_i\left(\frac{q_i}{1-q_i}-q_i\right)
\le2\sum_i p_iq_i^2\le\frac{2Q(w_k)^2}{p_*},
\]
normalizing these two sets of positive weights yields
\[
\left\|\frac{g(w_k)}{Q(w_k)}-F(w_k)\right\|\le\frac{4M}{p_*}Q(w_k).
\]
Also $\|W(t)-w_k\|\le M(t-t_k)$. Thus, for the discrepancy $d(t)=\dot W(t)-F(W(t))$,
\[
\begin{split}
\int_{t_j}^{\infty}\|d(t)\|\,dt
&\le\left(\frac{4M}{p_*}+\frac{\eta M^3}{2}\right)
\eta\sum_{k\ge j}Q(w_k)^2\\
&\le\left(\frac{4M}{p_*}+\frac{\eta M^3}{2}\right)\frac{2\delta}{\gamma^2}.
\end{split}
\]
Before $t_j$, both velocities have norm at most $M$ and $t_j\le\eta N_0$. The integral over that interval is at most $2M\eta N_0$. Nonexpansiveness gives
\[
\|W(t)-U(t)\|\le\int_0^t\|d(s)\|\,ds.
\]
This inequality follows by differentiating the distance, or a regularized distance when it is zero. Together these estimates prove~\eqref{eq:app-shadow-bound}.
The finite interval before small loss allows margins to become negative again after the first entry into $C$.
\end{proof}

\begin{lemma}[Reference-flow approximation]
\label{lem:app-reference-approximation}
For fixed $E\in C$, let $V_R(s)=U(Rs;RE)/R$.
On every fixed interval $[0,S]$,
\begin{equation}
\sup_{0\le s\le S}\|V_R(s)-x_E(s)\|
\le C_{E,S}\frac{\log(2+R)}R.
\label{eq:app-reference-log-bound}
\end{equation}
\par
All constants refer to the fixed path starting at $E$.
\end{lemma}

\begin{proof}[Proof]
It is the negative gradient flow of
\[
\Phi_R(x)=\frac1R\log\sum_i p_i e^{-Rz_i^\top x},
\qquad -b(x)-\frac{\log(1/p_*)}{R}\le\Phi_R(x)\le-b(x).
\]
Convexity, as in the correction comparison, gives
\[
\frac12\frac{d}{ds}\|V_R-x_E\|^2\le\frac{\log(1/p_*)}{R}.
\]
Thus the two flows starting at $E$ differ by at most $\sqrt{2S\log(1/p_*)/R}$ on $[0,S]$. The finite piecewise-affine structure of $x_E$ gives the stronger estimate~\eqref{eq:app-reference-log-bound}.

Consider one ideal segment, starting at $x_j=x_E(s_j)$, with velocity $v$ and fixed positive duration $\ell$. Set
\[
\xi_i=z_i-v,\qquad d_i=z_i^\top x_j-b(x_j),\qquad
\sigma_i=\xi_i^\top v.
\]
The gaps $d_i+s\sigma_i$ are nonnegative for $0\le s\le\ell$. In the moving coordinates
$h(t)=U(Rs_j+t;RE)-Rx_j-tv$, direct substitution gives
\begin{equation*}
\dot h(t)=-\nabla_h\psi_{R,t}(h(t)),\qquad
\psi_{R,t}(h)=\log\sum_i p_i e^{-Rd_i-\sigma_i t-\xi_i^\top h}.
\end{equation*}
The minimum margin $b(x_j)$ cancels. The argument therefore also applies when $b(E)=0$, including $E=0$.

Let $J=\{i:d_i=\sigma_i=0\}$, the indices that remain tied on this segment. Appendix~\ref{app:margin-geometry} gives $0\in\operatorname{conv}\{\xi_i:i\in J\}$. Take the smallest face of this convex hull containing zero, and let $J_0$ contain all indices whose vectors lie in that face. Zero lies in its relative interior. On $V=\operatorname{span}\{\xi_i:i\in J_0\}$, the function
\[
\phi_0(h)=\log\sum_{i\in J_0}p_i e^{-\xi_i^\top h}
\]
attains its minimum at some $h_*\in V$. To see this, a ball about zero in $V$ lies in $\operatorname{conv}\{\xi_i:i\in J_0\}$, so in every unit direction some $\xi_i$ with $i\in J_0$ has a uniformly negative inner product. This implies $\phi_0(h)\ge c\|h\|+\log p_*$ on $V$ for some $c>0$. If $V=\{0\}$, take $h_*=0$. Writing $m_* =\phi_0(h_*)$, we have
\[
\psi_{R,t}(h)\ge\phi_0(\Pi_V h)\ge m_*.
\]

A face of a finite polytope is exposed by a linear functional. Choose a vector $d$ orthogonal to $\xi_i$ for $i\in J_0$ and with $\xi_i^\top d>0$ for $i\in J\setminus J_0$. All $\xi_i$ with $i\in J$ are orthogonal to $v$, so $d$ may also be chosen orthogonal to $v$. Adding a sufficiently large multiple of $v$ makes it positive on every initially tied vector that immediately leaves the minimum. Rescaling gives
\[
\xi_i^\top d=0\ (i\in J_0),\qquad
\xi_i^\top d\ge3\ (i\in I(x_j)\setminus J_0).
\]
If the second set is empty, take $d=0$. Use the constant comparison point $h_R=h_*+d\log(2+R)$. There is a fixed $B\ge0$ such that
\begin{equation*}
0\le\psi_{R,t}(h_R)-m_*\le C(2+R)^{-3}
\quad\text{for }0\le t\le R\ell-B\log(2+R).
\end{equation*}
Indeed, the terms indexed by $J_0$ sum to $e^{m_*}$, and the other initially tied terms are suppressed by the choice of $d$. An initially inactive margin with positive gap at both ends has an unnormalized gap of order $R$ throughout the segment. If it reaches the minimum at the end, then $\sigma_i<0$, and on the shortened interval its unnormalized gap is at least $|\sigma_i|B\log(2+R)$. Choosing $B$ so that $|\sigma_i|B+\xi_i^\top d\ge3$ for each such index proves the bound. There are finitely many indices.

Convexity in $h$, with time held fixed, now gives
\[
\frac12\frac{d}{dt}\|h(t)-h_R\|^2
\le\psi_{R,t}(h_R)-\psi_{R,t}(h(t))\le C(2+R)^{-3}.
\]
Integrating over an interval of length $O(R)$ and using $\|h_R\|=O(\log(2+R))$ bounds $\|h(t)\|$ by $\|h(0)\|+C\log(2+R)$. The remaining interval has length $O(\log(2+R))$ and $\|\dot h\|\le2M$, so the same order holds on the full segment. Applying this estimate successively to the finitely many segments up to $S$ proves~\eqref{eq:app-reference-log-bound}. A final partial segment is handled by ending the comparison there.
\end{proof}

\paragraph{Post-entry comparison.}
Apply~\eqref{eq:app-shadow-bound} from $w_{\kappa_R}$, and compare its reference flow with the one from $RE$. Nonexpansiveness bounds their distance by $\|w_{\kappa_R}-RE\|$ at every common local time. With $X_R^+(s)=P_R(\tau_R+s)$, Lemmas~\ref{lem:app-reference-comparison} and~\ref{lem:app-reference-approximation} therefore imply
\begin{equation}
\sup_{0\le s\le S}\|X_R^+(s)-x_E(s)\|
\le\left\|\frac{w_{\kappa_R}}R-E\right\|
+\frac{C_0+C_{E,S}\log(2+R)}R.
\label{eq:app-post-entry-bound}
\end{equation}
This comparison uses the fixed ideal entry $E$ for the geometric estimate, so it requires no uniform geometric constants over a moving family of ideal initial points.

\subsection{Refined Error Bounds under Nonexpansive GD Updates}
\label{app:b5-logarithmic-correction}

We strengthen the correction estimate under~\eqref{eq:step-condition}. The main issue is a segment that lies on one or more zero-margin hyperplanes. Its ideal velocity may require a boundary point of the set of possible logistic contributions, so a finite offset need not reproduce that velocity exactly. We first show that an offset of logarithmic size approximates it to any prescribed polynomial accuracy.

\begin{lemma}[Approximation of correction velocities]
\label{lem:b5-logistic-balance}
For a fixed index set $\mathcal I\subseteq\{1,\ldots,n\}$, define
\[
K_{\mathcal I}=\left\{\sum_{i\in\mathcal I}p_i\theta_i z_i:
                   0\le\theta_i\le1\right\},
\qquad
G_{\mathcal I}(h)=\sum_{i\in\mathcal I}p_iq(z_i^\top h)z_i.
\]
For every fixed $c\in K_{\mathcal I}$ and $\nu>0$, there is a constant $C$ and, for every $R\ge1$, a vector $h_R$ such that
\begin{equation}
\|h_R\|\le C\log(2+R),
\qquad
\|G_{\mathcal I}(h_R)-c\|\le C(2+R)^{-\nu}.
\label{eq:b5-logistic-balance}
\end{equation}
The constant may depend on the vectors, weights, $\mathcal I$, $c$, and $\nu$, but not on $R$.
\end{lemma}

\begin{proof}
If $\mathcal I$ is empty, take $h_R=0$.

\paragraph{Relative-interior points.}
Fix an index set $J$ and write $V_J=\operatorname{span}\{z_i:i\in J\}$. Because $K_J$ contains zero, its affine hull is $V_J$. For every finite $h$, each coefficient $p_iq(z_i^\top h)$ lies strictly between zero and $p_i$. Such a coefficient vector maps into $\operatorname{ri}K_J$. Indeed, choose a basis of $V_J$ from the vectors $z_i$. Every sufficiently small vector in $V_J$ can be obtained by a sufficiently small change in the corresponding coefficients, which remain inside their intervals. Thus $G_J(h)$ has a relative neighborhood contained in $K_J$.

Conversely, let $d\in\operatorname{ri}K_J$. If $V_J=\{0\}$, then $d=0=G_J(0)$. Otherwise there is a $\rho>0$ such that
\[
d+\rho B_{V_J}\subset K_J,
\qquad
B_{V_J}=\{v\in V_J:\|v\|\le1\}.
\]
On $V_J$, consider the differentiable convex function
\[
\Psi_d(h)=\sum_{i\in J}p_i\log(1+e^{-z_i^\top h})+d^\top h.
\]
For $h\ne0$, the scalar inequality $\log(1+e^t)\ge[t]_+$ and the point $d-\rho h/\|h\|\in K_J$ give
\begin{align*}
\Psi_d(h)
&\ge \sum_{i\in J}p_i[-z_i^\top h]_++d^\top h
 =\max_{b\in K_J}(-b^\top h)+d^\top h
 \ge \rho\|h\|.
\end{align*}
Hence $\Psi_d$ is coercive on the finite-dimensional space $V_J$ and attains a minimum at some $h_0\in V_J$. Its full gradient $d-G_J(h_0)$ lies in $V_J$, so vanishing of its derivative on $V_J$ implies
\begin{equation}
G_J(h_0)=d.
\label{eq:b5-interior-match}
\end{equation}
This proves $G_J(\mathbb R^d)=\operatorname{ri}K_J$, without assuming that the vectors are linearly independent.

\paragraph{Boundary points.}
Now let $F$ be the smallest face of $K_{\mathcal I}$ containing $c$. Such a face exists because a polytope has finitely many faces and intersections of faces are faces. Moreover, $c\in\operatorname{ri}F$: otherwise a supporting hyperplane within $\operatorname{aff}F$ would produce a smaller face containing $c$.

We record explicitly why $F$ can be exposed by a single vector. If $F=K_{\mathcal I}$, take $\xi=0$. Otherwise put $U=\operatorname{span}(F-F)$ and project the vectors $r-c$, for vertices $r$ of $K_{\mathcal I}$ outside $F$, onto $U^\perp$. The convex hull $D$ of these projected vectors does not contain zero. To see this, a convex combination with zero projection would give a point of $K_{\mathcal I}\cap\operatorname{aff}F$ made from vertices outside $F$. But $K_{\mathcal I}\cap\operatorname{aff}F=F$: if $b$ were in that intersection, relative interior of $c$ would allow a point $c-\epsilon(b-c)\in F$, and the face property applied to their segment would imply $b\in F$. The face property would then require all vertices in the convex combination to belong to $F$, a contradiction.

Let $\xi$ be the minimum-norm point of $D$. Projection optimality gives a strictly positive inner product between $\xi$ and every projected vector in $D$, while $\xi$ is orthogonal to $U$. Therefore
\begin{equation*}
F=\arg\min_{b\in K_{\mathcal I}}\xi^\top b.
\end{equation*}
Here the construction of $D$ uses only the finitely many vertices of the polytope, so $D$ is compact and its nonzero projection is well defined.

Minimizing separately over the coefficient intervals in $K_{\mathcal I}$ now gives
\[
J=\{i\in\mathcal I:\xi^\top z_i=0\},
\qquad
b_-=\sum_{\substack{i\in\mathcal I\\\xi^\top z_i<0}}p_i z_i,
\qquad
F=b_-+K_J.
\]
Thus $d=c-b_-\in\operatorname{ri}K_J$. Choose $h_0$ from~\eqref{eq:b5-interior-match}; when $V_J=\{0\}$, use $h_0=0$. Set
\begin{equation}
h_R=h_0+\alpha\log(2+R)\,\xi.
\label{eq:b5-log-offset}
\end{equation}
The margins for indices in $J$ are unchanged. If $\xi=0$, all indices are in $J$ and the match is already exact. Otherwise let
\[
\sigma=\min_{i\in\mathcal I\setminus J}|\xi^\top z_i|>0
\]
and choose $\alpha\sigma\ge\nu$. For $\xi^\top z_i>0$,
\[
q(z_i^\top h_R)
\le e^{-z_i^\top h_0}(2+R)^{-\alpha\xi^\top z_i},
\]
whereas for $\xi^\top z_i<0$,
\[
1-q(z_i^\top h_R)
\le e^{z_i^\top h_0}(2+R)^{-\alpha|\xi^\top z_i|}.
\]
The finitely many remaining terms therefore approach their prescribed endpoints at order $(2+R)^{-\nu}$. Combining them with the exact contribution $G_J(h_0)=c-b_-$ proves the second bound in~\eqref{eq:b5-logistic-balance}. Equation~\eqref{eq:b5-log-offset} gives the first.
\end{proof}

\paragraph{A reference on each affine segment.}
Fix $T>0$.
By Appendix~\ref{app:finite-construction}, the correction flow has finitely many affine segments. Subdivide them at every additional zero-margin event in $[0,T]$. This introduces only finitely many cuts, since each nonconstant affine margin can vanish at most once on a straight segment. We obtain
\[
0=\sigma_0<\sigma_1<\cdots<\sigma_J=T
\]
such that, on each open interval $(\sigma_{j-1},\sigma_j)$, the velocity $v_j$ and the sets of negative, zero, and positive margins, denoted by $N_j,Z_j,P_j$, are constant. Include the constant segment after $\tau_*$ whenever it meets the window.

Set
\[
b_j^-=\sum_{i\in N_j}p_i z_i,
\qquad c_j=v_j-b_j^-.
\]
\par

\begin{lemma}[Reference trajectories on correction segments]
\label{lem:b5-segment-reference}
There are constants $C_h,C_d,B>0$, independent of $R$, and, for every $R\ge1$, offsets $h_{j,R}$ such that
\begin{equation}
\|h_{j,R}\|\le C_h\ell_R,
\qquad
\|G_{Z_j}(h_{j,R})-c_j\|\le C_h(2+R)^{-2},
\qquad \ell_R=\log(2+R).
\label{eq:b5-segment-balance}
\end{equation}
\par
The same $C_h$ can be used for all segments because their number is fixed. Define
\[
B_{j,R}(t)=Ry(t/R)+h_{j,R},
\qquad R\sigma_{j-1}\le t\le R\sigma_j.
\]
This reference is affine with derivative $v_j$.
\par

For this $B$, let
\[
I_{j,R}=[R\sigma_{j-1}+B\ell_R,\ R\sigma_j-B\ell_R].
\]
\par
These intervals are nonempty for sufficiently large $R$.
For all sufficiently large $R$,
\begin{equation}
\|g(B_{j,R}(t))-v_j\|\le C_d(2+R)^{-2}
\qquad(t\in I_{j,R}).
\label{eq:b5-reference-defect}
\end{equation}
\par
\end{lemma}

\begin{proof}[Proof]
The correction projection rule gives $c_j\in K_{Z_j}$. For every $i\in Z_j$, both $z_i^\top y(\tau)$ and $z_i^\top v_j$ vanish throughout the segment. Apply Lemma~\ref{lem:b5-logistic-balance} with target $c_j$ and $\nu=2$ to obtain~\eqref{eq:b5-segment-balance}.
We next keep the reference away from the segment endpoints. There is a $\mu>0$, independent of $R$ and $B$, such that every nonzero margin on segment $j$ obeys
\begin{equation*}
|z_i^\top Ry(t/R)|\ge\mu B\ell_R
\qquad(t\in I_{j,R},\ i\notin Z_j)
\end{equation*}
for sufficiently large $R$. To verify this, consider the fixed affine function $z_i^\top y(\tau)$ on $[\sigma_{j-1},\sigma_j]$. If it vanishes at an endpoint, its absolute value at distance at least $\delta$ from that endpoint is at least its nonzero slope magnitude times $\delta$. If neither endpoint is a zero, its absolute value has a positive lower bound on the entire interval. Apply these observations with $\delta=B\ell_R/R$ and take the minimum over the finitely many segment--sample pairs. A nonzero constant margin is covered by the second case.

Choose $B$ so that $\mu B-MC_h\ge2$. The offset then preserves all nonzero margin signs on $I_{j,R}$ and leaves their absolute values at least $2\ell_R$. The inequalities
\[
q(m)\le e^{-m}\quad(m\ge0),
\qquad
1-q(m)\le e^m\quad(m\le0)
\]
and~\eqref{eq:b5-segment-balance} imply~\eqref{eq:b5-reference-defect}.
Indeed, the negative-margin samples approximate $b_j^-$, the positive-margin samples contribute a vanishing amount, and the samples in $Z_j$ contribute $G_{Z_j}(h_{j,R})$.

This construction also applies to the constant segment after ideal entry. There $v_j=0$, $N_j$ is empty, and $c_j=0\in K_{Z_j}$. Thus~\eqref{eq:b5-reference-defect} controls the actual, unstopped GD motion during a fixed interval of original time of length $O(R)$ after the ideal correction has stopped.
\end{proof}

\begin{lemma}[Logarithmic correction-stage error bound]
\label{lem:b5-correction-envelope}
Assume~\eqref{eq:step-condition}, and let $\mathcal W_R(t)$ be the linear interpolation of the actual GD iterates at the original GD times $t=\eta k$, without stopping at entry. Extend the correction flow $y$ by its constant value $E$ after $\tau_*$. For every fixed $T>0$, there are constants $C_T$ and $R_0$ such that, whenever $R\ge R_0$ and $w_0=Ru_0+e_R$,
\begin{equation}
\sup_{0\le t\le RT}
\|\mathcal W_R(t)-Ry(t/R)\|
\le \|e_R\|+C_T\log(2+R).
\label{eq:b5-correction-envelope}
\end{equation}
The constants are independent of the perturbation $e_R$.
\end{lemma}

\begin{proof}
Write $g=-\nabla L$, so $g(w)=\sum_i p_iq(z_i^\top w)z_i$ and $\|g(w)\|\le M$.

\paragraph{Nonexpansion of one GD step.}
For every $w$,
\[
0\preceq\nabla^2L(w)
=\sum_i p_iq(z_i^\top w)(1-q(z_i^\top w))z_i z_i^\top
\preceq\Lambda I.
\]
For two points $u,v$, the matrix
\[
A_{u,v}=\int_0^1\nabla^2L(v+\theta(u-v))\,d\theta
\]
therefore has all eigenvalues in $[0,\Lambda]$. The GD map $T_\eta=I-\eta\nabla L$ satisfies
\[
T_\eta(u)-T_\eta(v)=(I-\eta A_{u,v})(u-v).
\]
When $0<\eta\Lambda\le2$, all eigenvalues of $I-\eta A_{u,v}$ lie in $[-1,1]$. Consequently,
\begin{equation}
\|T_\eta(u)-T_\eta(v)\|\le\|u-v\|.
\label{eq:b5-nonexpansion}
\end{equation}
This argument includes the endpoint $\eta\Lambda=2$.

We first prove~\eqref{eq:b5-correction-envelope} for $e_R=0$.

Use the reference trajectories from Lemma~\ref{lem:b5-segment-reference}.

\paragraph{Accumulating the discrete comparison.}
Let $t_k=\eta k$ and $K_R=\lfloor RT/\eta\rfloor$. Call a step good if its two endpoints belong to the same interval $I_{j,R}$. On such a step,
\[
B_{j,R}(t_{k+1})=B_{j,R}(t_k)+\eta v_j.
\]
Using~\eqref{eq:b5-nonexpansion} and~\eqref{eq:b5-reference-defect} yields
\begin{align*}
\|w_{k+1}-B_{j,R}(t_{k+1})\|
&\le \|T_\eta(w_k)-T_\eta(B_{j,R}(t_k))\|
      +\eta\|g(B_{j,R}(t_k))-v_j\|\\
&\le \|w_k-B_{j,R}(t_k)\|
      +C_d\eta(2+R)^{-2}.
\end{align*}
Thus each block of consecutive good steps increases the distance to its reference only by the sum of these one-step errors.

Write $D_k=\|w_k-Ry(t_k/R)\|$. At the two ends of a block, converting between its reference and $Ry$ costs at most $2C_h\ell_R$ in total. There are at most $J$ such blocks. All other steps lie within a distance $B\ell_R+\eta$ of one of the finitely many endpoints $R\sigma_j$. If $N_{\mathrm{bad}}$ is their number among $k=0,\ldots,K_R-1$, it follows that
\[
\eta N_{\mathrm{bad}}\le C_J(B\ell_R+\eta).
\]
On any such step, the speed bounds for GD interpolation and for $Ry(t/R)$ give
\[
D_{k+1}\le D_k+2M\eta.
\]
Starting from $D_0=0$ and applying these estimates chronologically gives, at every node through $K_R$,
\begin{align}
D_k
&\le 2JC_h\ell_R
    +2M\eta N_{\mathrm{bad}}
    +C_d\eta K_R(2+R)^{-2}
 \le C_T\ell_R.
\label{eq:b5-node-envelope}
\end{align}
The same bound applies inside a good block by summing the one-step errors up to that point.

For any $t\in[0,RT]$, choose its left GD node $t_k$. Both $\mathcal W_R$ and $Ry(t/R)$ have speed at most $M$, so their distances from their respective values at $t_k$ are at most $M\eta$. Equation~\eqref{eq:b5-node-envelope} therefore extends to the full interval, with an additional $2M\eta$ absorbed into $C_T\ell_R$.

\paragraph{Initialization perturbations.}
Let $\widetilde w_k$ denote the GD run starting exactly at $Ru_0$, and let $w_k$ start at $Ru_0+e_R$. Repeated application of~\eqref{eq:b5-nonexpansion} gives
\[
\|w_k-\widetilde w_k\|\le\|e_R\|
\qquad\text{for every }k.
\]
The same inequality holds for their linear interpolations because both use the same GD times. Adding this estimate to the bound for $\widetilde w_k$ proves~\eqref{eq:b5-correction-envelope}.
\end{proof}

To obtain the entry estimate, use Lemma~\ref{lem:b5-correction-envelope} on a fixed window extending beyond $\tau_*$. The normalized, unstopped correction error is bounded by
\[
\varepsilon_R=\frac{\|e_R\|+C_T\log(2+R)}{R}.
\]
Under~\eqref{eq:initialization}, this tends to zero. Applying Lemma~\ref{lem:app-entry-stability}, specifically~\eqref{eq:app-entry-joint}, gives
\begin{equation*}
\left\|\frac{w_{\kappa_R}}R-E\right\|+|\tau_R-\tau_*|
\le C\left(\frac{\|e_R\|}{R}
           +\frac{\log(2+R)}R\right).
\end{equation*}
Thus the correction-path, entry-position, and entry-time errors satisfy the improved bound. Combining it with~\eqref{eq:app-post-entry-bound} and accounting for the difference between entry times gives the complete-path estimate below.

\subsection{Completing the proof of Theorem~\ref{thm:path-convergence}}
\label{app:complete-convergence-proof}

\begin{proof}[Proof of Theorem~\ref{thm:path-convergence}]
Write
\[
\varepsilon_{1,R}(T)=\sup_{0\le\tau\le T}\|Y_R^{\mathrm{stop}}(\tau)-y(\tau)\|,
\qquad d_R=\|w_{\kappa_R}/R-E\|,
\qquad \Delta_R=|\tau_R-\tau_*|.
\]
Let $\varepsilon_{2,R}(T)$ denote the right-hand side of~\eqref{eq:app-post-entry-bound} with $S=T$. The two ideal stages have speed at most $M$. On a window where both paths are in correction, the error is at most $\varepsilon_{1,R}(T)$. Where both are in the second stage, comparison at the same local time costs $\varepsilon_{2,R}(T)$, and shifting from $r-\tau_R$ to $r-\tau_*$ costs at most $M\Delta_R$.

If GD enters first, then for $\tau_R\le r\le\tau_*$,
\[
\begin{split}
\|P_R(r)-y(r)\|
&\le\varepsilon_{2,R}(T)+\|x_E(r-\tau_R)-E\|+\|E-y(r)\|\\
&\le\varepsilon_{2,R}(T)+M\Delta_R.
\end{split}
\]
If GD enters later, then for $\tau_*\le r\le\tau_R$, the ideal correction flow is already fixed at $E$, so
\[
\|P_R(r)-x_E(r-\tau_*)\|
\le\varepsilon_{1,R}(T)+M(r-\tau_*)
\le\varepsilon_{1,R}(T)+M\Delta_R.
\]
Together these cases give
\begin{equation}
\sup_{0\le r\le T}\|P_R(r)-P(r)\|
\le\max\{\varepsilon_{1,R}(T),\varepsilon_{2,R}(T)\}+M\Delta_R.
\label{eq:app-path-bound}
\end{equation}
Lemmas~\ref{lem:app-correction-approximation} and~\ref{lem:app-entry-stability} bound $\varepsilon_{1,R}$, $d_R$, and $\Delta_R$ by $C(\delta_R+R^{-1/2})$. Inserting them and~\eqref{eq:app-post-entry-bound} into~\eqref{eq:app-path-bound} proves the first part of Theorem~\ref{thm:path-convergence}. Under~\eqref{eq:step-condition}, combining Lemmas~\ref{lem:b5-correction-envelope} and~\ref{lem:app-entry-stability} improves all three entry quantities to $C(\delta_R+\log(2+R)/R)$. The same connection estimate proves the second part. The argument also covers $\tau_*=0$ and $E=0$.
\end{proof}

\section{Proofs for Path Functionals}
\label{app:functionals}

We prove the results of Section~\ref{sec:path-functionals} under the assumptions of Theorem~\ref{thm:path-convergence}. In particular, the data, positive sample weights, and step size are fixed, and $w_0=Ru_0+e_R$ with $\|e_R\|/R\to0$. We retain the actual entry node $\kappa_R$ and the two-stage parameter $r_{R,k}$ throughout. Fix $S>\tau_*$ and let $N_R=\min\{k:r_{R,k}\ge S\}$, as in the main text.

Write $g(w)=\sum_i p_iq(z_i^\top w)z_i$ and $M=\max_i\|z_i\|$. Since $0<Q(w)\le1$ and $\|g(w)\|\le MQ(w)$, the actual path has speed at most $M$ in both stages. Its node intervals satisfy
\begin{equation}
\Delta r_{R,k}=\begin{cases}
\eta/R,&k<\kappa_R,\\
\eta Q(w_k)/R,&k\ge\kappa_R,
\end{cases}
\qquad
0\le r_{R,N_R}-S\le\frac\eta R.
\label{eq:app-functional-mesh}
\end{equation}
Thus the entire stopped sequence lies within a fixed slightly longer window for sufficiently large $R$. Set
\[
\varepsilon_R=\sup_{0\le r\le S+1}\|P_R(r)-P(r)\|.
\]
The path bound~\eqref{eq:app-path-bound} gives $\varepsilon_R\to0$, retaining the initialization error $\|e_R\|/R$.
\par

\subsection{Convergence of Path Functionals}
\label{app:functional-transfer}

\begin{proof}[Proof of Proposition~\ref{prop:functional-transfer}]
The trajectory $\Gamma_R=((w_k,\eta k))_{k=0}^{N_R}$ includes the last node, whereas the continuous path used in Proposition~\ref{prop:functional-transfer} is $P_R$ restricted to $[0,S]$. For a functional $\mathcal F$ and a positive scale $a_R$, define its approximation error by
\[
\beta_R=\left|\frac{\mathcal F[\Gamma_R]}{a_R}-\Phi[P_R]\right|.
\]
If $\beta_R\to0$ and $\Phi$ is continuous at $P$ in the uniform norm, then
\begin{equation*}
\left|\frac{\mathcal F[\Gamma_R]}{a_R}-\Phi[P]\right|
\le\beta_R+|\Phi[P_R]-\Phi[P]|\longrightarrow0.
\end{equation*}
This proves Proposition~\ref{prop:functional-transfer}. If $\omega_\Phi$ is a local modulus of continuity at $P$, the right-hand side is at most $\beta_R+\omega_\Phi(\varepsilon_R)$. For a locally $K$-Lipschitz functional, it is at most $\beta_R+K\varepsilon_R$.
\end{proof}

\subsection{Maximum values and the finite-vertex formula}
\label{app:maximum-proof}

\begin{proof}[Proof of Corollary~\ref{cor:maximum-functional}]
Assume $f(Rx)/a_R\to f_\infty(x)$ locally uniformly, with $f_\infty$ continuous. Choose a compact ball $K$ containing the images of $P$ and $P_R$ on $[0,S+1]$ for all sufficiently large $R$. Such a ball exists by uniform path convergence. Let
\[
\rho_R=\sup_{x\in K}\left|\frac{f(Rx)}{a_R}-f_\infty(x)\right|,
\]
and let $\omega_K$ be a modulus of continuity of $f_\infty$ on $K$.

For any stopped node, put $r'=\min\{r_{R,k},S\}$. The ideal path is $M$-Lipschitz, so~\eqref{eq:app-functional-mesh} gives
\[
\left\|\frac{w_k}{R}-P(r')\right\|
\le\varepsilon_R+\frac{M\eta}{R}.
\]
This bounds the scaled maximum over the nodes from above by the ideal maximum plus $\rho_R+\omega_K(\varepsilon_R+M\eta/R)$. Conversely, choose $r_*\in[0,S]$ attaining the ideal maximum and a node bordering the interval containing $r_*$. This node belongs to $0,\ldots,N_R$ and lies at parameter distance at most $\eta/R$ from $r_*$. The same position estimate gives the reverse bound. Hence
\begin{equation}
\left|\frac1{a_R}\max_{0\le k\le N_R}f(w_k)
-\max_{0\le r\le S}f_\infty(P(r))\right|
\le\rho_R+\omega_K\left(\varepsilon_R+\frac{M\eta}{R}\right).
\label{eq:app-maximum-error}
\end{equation}
Both terms tend to zero.

This also verifies the representation required by Proposition~\ref{prop:functional-transfer}. For $\Phi[\gamma]=\max_{r\in[0,S]}f_\infty(\gamma(r))$, the same node comparison with $P_R$ itself gives approximation error at most $\rho_R+\omega_K(M\eta/R)$. Continuity of $f_\infty$ on a compact neighborhood of a path's image implies continuity of $\Phi$ in the uniform norm.
\end{proof}

\paragraph{Logistic evaluation loss.}
For the fixed evaluation features and weights in~\eqref{eq:evaluation-loss}, set $A_{\mathrm{ev}}=\sum_j\rho_j\|a_j\|$. The scalar identity
\[
\log(1+e^{-m})=[-m]_++\log(1+e^{-|m|})
\]
implies the global comparison
\begin{equation*}
0\le\frac{L_{\mathrm{ev}}(Rx)}R-H_{\mathrm{ev}}(x)
\le\frac{\log2}{R}.
\end{equation*}
The function $H_{\mathrm{ev}}$ is convex and $A_{\mathrm{ev}}$-Lipschitz. Applying~\eqref{eq:app-maximum-error} therefore yields
\begin{equation}
\left|\frac1R\max_{0\le k\le N_R}L_{\mathrm{ev}}(w_k)
-\max_{0\le r\le S}H_{\mathrm{ev}}(P(r))\right|
\le A_{\mathrm{ev}}\varepsilon_R+
\frac{A_{\mathrm{ev}}M\eta+\log2}{R}.
\label{eq:app-evaluation-peak-error}
\end{equation}
The normalized peak thus inherits the path error rate, including the initialization term and the strengthened rate under~\eqref{eq:step-condition}.

Let $0=t_0<t_1<\cdots<t_m=S$ include every change of ideal velocity in $(0,S)$ and the connection $\tau_*$ when it is interior. On each interval $[t_j,t_{j+1}]$, the path is affine, so $H_{\mathrm{ev}}\circ P$ is convex. Its value is bounded above by the larger endpoint value. Consequently,
\begin{equation*}
\max_{0\le r\le S}H_{\mathrm{ev}}(P(r))
=\max_{0\le j\le m}H_{\mathrm{ev}}(P(t_j)).
\end{equation*}
Evaluation-margin zero crossings need not be added to this list. The finite list suffices for the limiting value; the estimate does not identify the maximizing GD node. Convexity of $L_{\mathrm{ev}}$ also shows that its maximum on the entire interpolated unnormalized path over $[0,r_{R,N_R}]$ equals its maximum over the stopped nodes exactly.

\subsection{Cumulative training loss in GD time}
\label{app:cumulative-proof}

We prove Corollary~\ref{cor:cumulative-loss} using the time parameterization in each stage. A uniform estimate after entry into $C$ controls the difference between $L$ and $Q$.

\begin{lemma}[Uniform accumulated loss difference after entry]
\label{lem:app-cumulative-loss-gap}
First run the same GD updates from any unnormalized point $v_0\in C$, and denote the iterates by $v_k$.
Let $N_0$ and $\delta$ be as in Lemma~\ref{lem:app-small-loss}.
Then
\begin{equation}
z_i^\top v_k\ge-B_0\quad\text{for all }i,k,
\qquad B_0=\eta M^2N_0.
\label{eq:app-functional-margin-bound}
\end{equation}
\par
The same lower bound holds between nodes.
\begin{equation}
0\le\eta\sum_{k=0}^{\infty}\bigl(L(v_k)-Q(v_k)\bigr)
\le C_{\mathrm{gap}},
\qquad
C_{\mathrm{gap}}=\eta N_0(B_0+\log2)+\frac{2\delta}{\gamma^2p_*}.
\label{eq:app-cumulative-loss-gap}
\end{equation}
\par
The constants depend only on the fixed data, weights, and step size.
\end{lemma}

\begin{proof}[Proof]
Lemma~\ref{lem:app-small-loss} gives an index $j_0\le N_0$ such that
\[
L(v_k)\le\delta\quad(k\ge j_0),
\qquad
\delta=\min\left\{\frac{p_*}{2},\frac1{4\eta M^2}\right\}.
\]
The small-loss region is invariant by~\eqref{eq:app-small-loss-descent}, and~\eqref{eq:app-q-square-sum} gives
\[
\eta\sum_{k\ge j_0}Q(v_k)^2\le\frac{2\delta}{\gamma^2}.
\]
Every individual sample loss in this region is at most $\delta/p_*\le1/2<\log2$, so all margins are strictly positive there. Before $j_0$, each update moves the parameter by at most $\eta M$. Since the starting margins are nonnegative, this proves~\eqref{eq:app-functional-margin-bound}. Thus this argument includes any temporary negative margins after the first nonnegative node.

For $q=q(m)$, the scalar loss is $\ell(m)=-\log(1-q)$. Hence $\ell(m)-q\ge0$ for every margin. If $0\le q\le1/2$, then
\[
0\le-\log(1-q)-q=\sum_{j=2}^{\infty}\frac{q^j}{j}
\le\frac{q^2}{2(1-q)}\le q^2.
\]
For $k\ge j_0$, write $q_i=q(z_i^\top v_k)$. Since $q_i\le Q(v_k)/p_*$,
\[
0\le L(v_k)-Q(v_k)
\le\sum_i p_iq_i^2\le\frac{Q(v_k)^2}{p_*}.
\]
For $k<j_0$, equation~\eqref{eq:app-functional-margin-bound} gives $L(v_k)\le B_0+\log2$. Summing the two parts proves~\eqref{eq:app-cumulative-loss-gap}.
\end{proof}
We can therefore apply the estimate at the unnormalized entry point $v_0=w_{\kappa_R}$ without requiring an $O(1)$ entry perturbation.

\IdealPathCumulativeProofSpace
\begin{proof}[Proof of Corollary~\ref{cor:cumulative-loss}]
\leavevmode\par
\paragraph{The exact time change.}
By the entry estimate~\eqref{eq:app-entry-joint}, $\tau_R\to\tau_*$, and thus $\tau_R<S$ and $N_R>\kappa_R$ for sufficiently large $R$. Define the left-node path by $\underline P_R(r)=w_k/R$ on $[r_{R,k},r_{R,k+1})$. It satisfies
\[
\|\underline P_R(r)-P_R(r)\|\le\frac{\eta M}{R}.
\]
Using the node durations in~\eqref{eq:app-functional-mesh}, the definitions of $C_R^-$ and $C_R^+$ give the exact identities
\begin{equation*}
\begin{aligned}
\frac{C_R^-}{R^2}
&=\int_0^{\tau_R}\frac{L(R\underline P_R(r))}{R}\,dr,\\
\frac{C_R^+}{R}
&=\int_{\tau_R}^{r_{R,N_R}}
\frac{L(R\underline P_R(r))}{Q(R\underline P_R(r))}\,dr.
\end{aligned}
\end{equation*}
The integrands use the loss at the left node of each update. The final update is included in full, but there is no extra term $L(w_{N_R})$.

\paragraph{Correction and the full-path functional.}
The bound~\eqref{eq:loss-hinge-comparison} and the $M$-Lipschitz property of $H$ imply
\[
\left|\frac{L(R\underline P_R(r))}{R}-H(P_R(r))\right|
\le\frac{\log2+\eta M^2}{R}.
\]
After actual entry, equation~\eqref{eq:app-functional-margin-bound} gives $H(P_R(r))\le B_0/R$. Applying these bounds before and after $\tau_R$ yields
\begin{equation}
\left|\frac{C_R^-}{R^2}-\int_0^S H(P_R(r))\,dr\right|
\le\frac{S(\log2+\eta M^2+B_0)}{R}.
\label{eq:app-correction-area-compatibility}
\end{equation}
The functional $\Phi[\gamma]=\int_0^S H(\gamma(r))\,dr$ satisfies
\[
|\Phi[\gamma]-\Phi[\widetilde\gamma]|
\le MS\|\gamma-\widetilde\gamma\|_\infty.
\]
Since $H(P(r))=0$ for $r\ge\tau_*$, equation~\eqref{eq:app-correction-area-compatibility} proves the first limit in~\eqref{eq:cumulative-loss-limits}. More precisely, its error is at most
\[
MS\varepsilon_R+\frac{S(\log2+\eta M^2+B_0)}R.
\]

\paragraph{The contribution after entry.}
The second-stage time parameterization gives
\[
\eta\sum_{k=\kappa_R}^{N_R-1}Q(w_k)
=R(r_{R,N_R}-\tau_R).
\]
The uniform bound~\eqref{eq:app-cumulative-loss-gap} in Lemma~\ref{lem:app-cumulative-loss-gap} therefore implies
\[
0\le\frac{C_R^+}{R}-(r_{R,N_R}-\tau_R)
\le\frac{C_{\mathrm{gap}}}{R}.
\]
Combining this with the last-node bound gives
\begin{equation*}
\left|\frac{C_R^+}{R}-(S-\tau_*)\right|
\le|\tau_R-\tau_*|+\frac{C_{\mathrm{gap}}+\eta}{R}
\longrightarrow0.
\end{equation*}
This proves the second limit. It also gives $C_R^+/R^2=O(1/R)$; together with~\eqref{eq:app-correction-area-compatibility},
\begin{equation*}
\frac{C_R}{R^2}=\Phi[P_R]+O(R^{-1})\longrightarrow\Phi[P].
\end{equation*}
This verifies the full cumulative functional's approximation in Proposition~\ref{prop:functional-transfer}. Each stage's normalized error inherits the path and entry rates from~\eqref{eq:app-path-bound} and~\eqref{eq:app-entry-joint}.
\end{proof}

\paragraph{Computing the correction coefficient.}
The function $H$ is positively homogeneous of degree one. For any $\xi\in\partial H(x)$, applying the subgradient inequality at $0$ and $2x$ gives $\xi^\top x=H(x)$. The absolutely continuous correction flow satisfies $\dot y=-\xi$ almost everywhere, so
\[
\frac{d}{dr}\frac{\|y(r)\|^2}{2}=-H(y(r))
\quad\text{almost everywhere}.
\]
Integrating up to the finite correction endpoint yields
\begin{equation}
\int_0^{\tau_*}H(y(r))\,dr
=\frac{\|u_0\|^2-\|E\|^2}{2}.
\label{eq:app-correction-norm-identity}
\end{equation}
Here $E$ is the correction endpoint, not the final point $P(S)$ of the full window. The identity applies also when $u_0$ is not a unit vector or $\tau_*=0$.

\paragraph{Zero-duration correction and the two scales.}
If $\tau_*=0$, the limits remain $C_R^-/R^2\to0$ and $C_R^+/R\to S$. A vanishing first limit does not imply $C_R^-=o(R)$. For example, with $z_1=p_1=1$, $u_0=0$, and $e_R=-R^{3/4}$, use $\widetilde R=R^{3/4}$ as the initialization scale. The correction path from $-1$ is $-1+t$ for $0\le t\le1$, whose negative-margin integral is $1/2$. The correction limit at this scale gives $C_R^-\sim R^{3/2}/2$. If instead every margin at $u_0$ is strictly positive, then $\kappa_R=0$ for sufficiently large $R$, so $C_R/R\to S$. In general, the two stage limits describe their respective leading scales; they do not supply an $o(R)$ remainder for the total cumulative loss after subtracting both leading terms.

\section{Analytic Examples}
\label{app:analytic-examples}

\subsection{The two-stage path in Figure~\ref{fig:two-stage-path}}
\label{app:two-sample-path}

Take $z_1=(2,0)$, $z_2=(1,2)$, $p_1=p_2=1/2$, and $u_0=(-1,-1/5)$. The data are strictly separable, since $(1,1)$ has positive inner product with both features. Applying the projection rules gives the complete path
\begin{equation}
P(r)=\begin{cases}
(-1+3r/2,\,-1/5+r),&0\le r\le2/5,\\
(r-4/5,\,1/5),&2/5\le r\le4/5,\\
(2r-8/5,\,1/5),&4/5\le r\le1,\\
((8r-6)/5,\,(4r-3)/5),&r\ge1.
\end{cases}
\label{eq:app-two-sample-path}
\end{equation}

Initially both margins are negative, so the correction velocity is $(z_1+z_2)/2=(3/2,1)$. The second margin first reaches zero at $r=2/5$, at the point $(-2/5,1/5)$. There the correction projection minimizes
\[
\left\|\frac{z_1}{2}+\frac\theta2z_2\right\|^2
=(1+\theta/2)^2+\theta^2,\qquad 0\le\theta\le1.
\]
Its minimum is at $\theta=0$, giving velocity $(1,0)$. The second margin immediately becomes positive, and the first reaches zero at $\tau_*=4/5$. Thus $E=(0,1/5)$.

After entry the first margin is uniquely smallest, so the continuation has velocity $z_1=(2,0)$ until $r=1$. The margins then tie at $(2/5,1/5)$. A point in their convex hull is $\alpha z_1+(1-\alpha)z_2=(1+\alpha,2-2\alpha)$. Its squared norm is minimized at $\alpha=3/5$, giving $v=(8/5,4/5)$. Since
\[
z_1^\top v=z_2^\top v=16/5,
\]
both margins remain tied for all subsequent time.

For completeness, the training margins on these segments are
\begin{equation*}
\bigl(z_1^\top P(r),z_2^\top P(r)\bigr)=
\begin{cases}
(-2+3r,\,-7/5+7r/2),&0\le r\le2/5,\\
(2r-8/5,\,r-2/5),&2/5\le r\le4/5,\\
(4r-16/5,\,2r-6/5),&4/5\le r\le1,\\
((16r-12)/5,\,(16r-12)/5),&r\ge1.
\end{cases}
\end{equation*}
These formulas give the correction turn, entry, and persistent tie shown in Figure~\ref{fig:two-stage-path}. They use the two-stage parameter $r$ from~\eqref{eq:complete-path}.

$H(P(r))$ equals $17/10-13r/4$ on $[0,2/5]$ and $4/5-r$ on $[2/5,4/5]$. Its integral is
\begin{equation*}
\int_0^{4/5}H(P(r))\,dr
=\frac{21}{50}+\frac{2}{25}=\frac12
=\frac{26/25-1/25}{2},
\end{equation*}
in agreement with~\eqref{eq:app-correction-norm-identity}.

\subsection{The evaluation-loss peak in Figure~\ref{fig:functional-example}}
\label{app:peak-example}

Use the three-dimensional training features $z_1=(2,0,0)$ and $z_2=(1,2,0)$ with the same weights, and set $u_0=(-1,-1/5,1)$. The vector $(1,1,0)$ strictly separates the training data. Every update has zero third coordinate, so the actual third coordinate remains $R+(e_R)_3$. The ideal path is obtained by appending one to every position in~\eqref{eq:app-two-sample-path}.

Take the single evaluation feature $a=(-2/5,1,-2/25)$ with unit weight and set $S=9/5$. This feature measures $L_{\mathrm{ev}}(w)=\log(1+e^{-a^\top w})$ and does not enter the updates. Its margin along the ideal path is
\begin{equation*}
m(r)=a^\top P(r)=\begin{cases}
3/25+2r/5,&0\le r\le2/5,\\
11/25-2r/5,&2/5\le r\le4/5,\\
19/25-4r/5,&4/5\le r\le1,\\
(4r-5)/25,&1\le r\le9/5.
\end{cases}
\end{equation*}
At $r=0,2/5,4/5,1,9/5$, the respective margins are
\[
\frac3{25},\qquad\frac7{25},\qquad\frac3{25},\qquad-\frac1{25},\qquad\frac{11}{125}.
\]
Thus the negative-margin interval is exactly $(19/20,5/4)$, and the limiting evaluation loss is
\begin{equation*}
H_{\mathrm{ev}}(P(r))=\begin{cases}
(20r-19)/25,&19/20\le r\le1,\\
(5-4r)/25,&1\le r\le5/4,\\
0,&r\in[0,19/20]\cup[5/4,9/5].
\end{cases}
\end{equation*}
Its unique maximum is $1/25$ at $r=1$. Within the invariant plane $x_3=1$, the evaluation boundary is $x_2=(2/5)x_1+2/25$. The path starts above this line, crosses below it, and returns above it, as shown in Figure~\ref{fig:functional-example}.

\paragraph{Consequences for actual GD.}
Here $A_{\mathrm{ev}}=\|a\|=27/25$ and $M=\sqrt5$. Equation~\eqref{eq:app-evaluation-peak-error} gives
\[
\max_{0\le k\le N_R}L_{\mathrm{ev}}(w_k)
=\frac R{25}+O\bigl(\|e_R\|+\sqrt R+1\bigr)
=\frac R{25}+o(R).
\]
Under~\eqref{eq:step-condition}, $\sqrt R$ is replaced by $\log(2+R)$. The ideal endpoint margins are $3/25$ and $11/125$, both strictly positive. Since $w_0/R\to u_0$ and $w_{N_R}/R\to P(S)$, both unnormalized endpoint margins are at least $11R/250$ for sufficiently large $R$. The inequality $\log(1+e^{-m})\le e^{-m}$ then gives
\begin{equation*}
L_{\mathrm{ev}}(w_0),\ L_{\mathrm{ev}}(w_{N_R})
\le e^{-11R/250}\longrightarrow0.
\end{equation*}
This establishes the intermediate peak and endpoint recovery for every allowed perturbation $e_R=o(R)$.

Appending a constant coordinate leaves the norm difference in~\eqref{eq:app-correction-norm-identity} unchanged. For this same window, the two cumulative coefficients are therefore $1/2$ and $S-\tau_*=1$:
\[
\frac{C_R^-}{R^2}\longrightarrow\frac12,
\qquad
\frac{C_R^+}{R}\longrightarrow1.
\]

\section{Experimental Details and Additional Results}
\label{app:experiments}

This appendix specifies the numerical procedures and data used in Section~\ref{sec:experiments}, and reports supplementary results for the same runs. The ideal paths are constructed from the weighted training data and initial position, independently of the GD trajectories.

\subsection{Numerical methods}
\label{app:numerical-methods}

\paragraph{GD, entry, and stopping.}
All experiments use fixed-step, full-batch logistic GD in double precision. We evaluate $q(m)$ and the loss without forming a large exponential:
\[
q(m)=\begin{cases}e^{-m}/(1+e^{-m}),&m\geq0,\\
1/(1+e^m),&m<0,\end{cases}
\qquad
\ell(m)=[-m]_++\log(1+e^{-|m|}).
\]
The compiled update loop uses compensated summation for the parameter updates, path parameter, and cumulative losses, with floating-point reassociation disabled. Entry is the first actual node with every computed training margin nonnegative, with no margin tolerance. We switch to the second-stage time parameter at this iterate:
\begin{equation*}
r_k=\begin{cases}
\eta k/R,&k\leq\kappa_R,\\
\eta\kappa_R/R+\dfrac{\eta}{R}\displaystyle\sum_{j=\kappa_R}^{k-1}Q(w_j),&k>\kappa_R.
\end{cases}
\end{equation*}
The main-window runs stop at $N_R=\min\{k:r_k\geq S\}$. Their interpolated paths use the points $(r_k,w_k/R)$, including the interval containing $S$. The entry-location experiments instead stop at $\kappa_R$ and measure $\|w_{\kappa_R}/R-E\|$; they do not require an actual continuation window.

\paragraph{Path and functional measurements.}
The reported path error is the full Euclidean norm in the original feature space,
\[
\mathcal E_R(S)=\sup_{0\leq r\leq S}\|P_R(r)-P(r)\|.
\]
On each interval between consecutive vertices of the two paths, their difference is affine and its norm is convex. It therefore suffices to check every GD node in $[0,S]$, every ideal vertex in $[0,S]$ with the GD path interpolated there, and the interpolated point at $S$. In particular, ideal vertices inside the final GD interval are included even when its right endpoint exceeds $S$. There is no alignment or translation of entry times.

Evaluation maxima are tracked at every actual node $0\leq k\leq N_R$, including the stopping node. Cumulative training losses are tracked before each update, with
\[
C_R^-=\eta\sum_{k<\min\{\kappa_R,N_R\}}L(w_k),
\qquad
C_R^+=\eta\sum_{\kappa_R\leq k<N_R}L(w_k),
\]
where the second sum is empty if entry has not occurred. Thus the stopping node contributes to the maximum but not to either cumulative sum. The plotted trajectories connect saved nodes, whereas these statistics use every update. Synthetic runs retain the first $5000$ nodes, subsequent nodes separated by at least $S/4000$ in the path parameter, and the entry node, subject to a $60000$-record limit. The corresponding digits settings are $2000$, $S/3000$, and $12000$. A plotted curve is not extrapolated beyond its last saved node.

\paragraph{Independent ideal-path construction.}
We implement the finite event construction of Appendix~\ref{app:ideal-path-construction}. During correction, write $c=\sum_{i:z_i^\top x<0}p_i z_i$ and let $W$ have columns $p_i z_i$ for the zero-margin samples. The velocity is $c+W\theta$, where
\[
\min_{0\leq\theta\leq1}\frac12\|c+W\theta\|^2
\]
is solved by bounded-variable least squares (BVLS), with tolerance $10^{-15}$ and at most $10000$ iterations. The next event is the first positive zero-crossing time of a currently nonzero margin. During continuation, we minimize $\|\sum_{i\in I(x)}\lambda_i z_i\|^2/2$ over the simplex. Active sets of at most eight samples use support enumeration up to size $d+1$; otherwise, or if its check fails, we use Wolfe's minimum-norm algorithm. The next event is the first equality between the current minimum margin and an inactive margin.

The event tolerance is $10^{-11}$. Margin signs and ties use this tolerance multiplied by
$\max\{1,\max_i|z_i^\top u_0|,\max_i\|z_i\|\}$; candidate event times must exceed $10^{-11}$. The continuation solver uses tolerance $10^{-11}$, checks the simplex and inequalities $z_i^\top v\geq\|v\|^2$, tries its fallback at tolerance $10^{-9}$, and rejects a result that fails the final $10^{-7}$ check. The synthetic construction extends to parameter $50$ with at most $10000$ events; the digits construction extends to $10000$ with at most $50000$ events. All displayed windows lie within these constructions.

We additionally checked every correction subproblem encountered when rebuilding the $46$ distinct instances selected for the manuscript. For $g=W^\top(c+W\theta)$, the maximum residual
$\|\theta-\operatorname{clip}_{[0,1]}(\theta-g)\|_\infty$ was $1.12\times10^{-16}$, and the maximum box duality gap
$g^\top\theta-\sum_j\min\{g_j,0\}$ was $2.80\times10^{-17}$. All coefficients were feasible. There were $590$ calls, $544$ with a nonempty boundary set, of size at most five. Rebuilt correction vertices differed from their saved counterparts by at most $1.74\times10^{-14}$, and entry times by at most $2.49\times10^{-14}$. These computations check floating-point consistency and the optimality conditions of the projection subproblems. All saved trajectory arrays were finite; the smallest recorded $Q$ was $6.26\times10^{-11}$.

\paragraph{Window selection and computational cost.}
The windows are fixed before varying $R$. For synthetic data, we form $400001$ equally spaced points on $[\tau_*,50]$ and select the first at which the ideal minimum margin is at least $0.2$ for the controlled examples or $0.15$ for the random family. For digits, the same rule uses $[\tau_*,10000]$ and threshold $0.15$. We use the same grid-selected window at every value of $R$.

The GD runs and independent reproducibility checks used Python~3.13.3, NumPy~2.2.6, SciPy~1.15.3, and scikit-learn~1.7.0; the figures were drawn with Python~3.13.5, NumPy~2.5.3, and Matplotlib~3.11.2, on macOS ARM64. The GD implementation uses a Numba-compiled update loop. Table~\ref{tab:app-runtime} gives the original recorded costs of the selected runs. Times are sums of per-run wall times, not elapsed time for a parallel batch. The additional reconstruction of the ideal paths took $0.11$ seconds and did not rerun GD.

\begin{table}[!htbp]
\centering
\small
\caption{Recorded computation for the selected numerical results. The last column is the allowed number of updates per run; every run in this table completed its specified stopping rule.}
\label{tab:app-runtime}
\begin{tabular}{lrrrr}
\hline
Runs & Number & Total seconds & Largest update count & Budget per run\\
\hline
G1, short window & 5 & 0.63 & $5.395\times10^6$ & $4\times10^9$\\
G2, main window & 5 & 8.09 & $5.962\times10^7$ & $4\times10^9$\\
Random instances & 200 & 187.93 & $1.415\times10^9$ & $2\times10^9$\\
32-image digits task & 5 & 452.75 & $2.691\times10^7$ & $2\times10^8$\\
Larger digits tasks, entry & 12 & 382.82 & $4.785\times10^6$ & $5\times10^7$\\
\hline
\end{tabular}
\end{table}

\subsection{Controlled geometries}
\label{app:controlled-geometries}

Both controlled examples use $R\in\{5,10,20,40,80\}$ and $w_0=Ru_0$. The first, denoted G1, embeds the two-sample geometry in Appendix~\ref{app:two-sample-path} into three coordinates:
\[
Z=\begin{pmatrix}2&0&0\\1&2&0\end{pmatrix},\quad
p=(1/2,1/2),\quad u_0=(-1,-1/5,1),\quad\eta=1.
\]
The vector $(1,0,0)$ strictly separates the training data, and
$\Lambda=(9+\sqrt{17})/16$. The third coordinate stays equal to one after normalization and is included in the path-error calculation. The short window is
$S=0.8500610000000001$, with $\tau_*=4/5$. Its vertices are
$(-1,-1/5,1)$ at $r=0$, $(-2/5,1/5,1)$ at $r=2/5$, $(0,1/5,1)$ at $r=4/5$, and $(2(S-4/5),1/5,1)$ at the endpoint. This window tests the training path before the later evaluation peak of Appendix~\ref{app:peak-example}. Its path errors at the five values of $R$ are $0.12623$, $0.06744$, $0.03526$, $0.01770$, and $0.00885$.

The full-rank example, G2, has
\begin{equation*}
Z=\begin{pmatrix}
2&1&3\\-2&3&1\\2&-2&1\\-3&3&1
\end{pmatrix},\qquad
p=(1/4,1/4,1/4,1/4),\qquad
u_0=(0,-3/4,-1/2),\qquad\eta=1/2.
\end{equation*}
The separator $(0,0,1)$ gives margins $(3,1,1,1)$, and
$\Lambda=2.4615063831810273$, so $\eta\Lambda=1.2307531915905137$. Its fixed window is
$S=0.6224902717610965$. Table~\ref{tab:app-g2-construction} determines the ideal path over the entire window without numerical integration: start at $x_0=u_0$ and use
$x_{j+1}=x_j+(t_{j+1}-t_j)v_j$, with $P(r)=x_j+(r-t_j)v_j$ between successive times. The first five rows are correction and the last three are continuation.

\begin{table}[!htbp]
\centering
\small
\caption{Finite construction of G2 on the reported window. All event times and velocities shown as fractions are exact; the final time is the original grid-selected $S$.}
\label{tab:app-g2-construction}
\begin{tabular}{crrl}
\hline
$j$ & $t_j$ & $t_{j+1}$ & $v_j$\\
\hline
0 & $0$ & $4/15$ & $(-3/4,7/4,5/4)$\\
1 & $4/15$ & $121/360$ & $(-1/4,5/4,3/2)$\\
2 & $121/360$ & $151/360$ & $(1/2,1/2,5/4)$\\
3 & $151/360$ & $383/864$ & $(2/9,7/9,10/9)$\\
4 & $383/864$ & $28369/54432$ & $(1/2,1/4,3/4)$\\
5 & $28369/54432$ & $85637/163296$ & $(2,1,3)$\\
6 & $85637/163296$ & $43375/81648$ & $(0,2,2)$\\
7 & $43375/81648$ & $S$ & $(10/41,8/41,1)$\\
\hline
\end{tabular}
\end{table}

In particular,
\[
\tau_*=\frac{28369}{54432}\simeq0.5211823927101705,\qquad
E=\left(-\frac{265}{2016},-\frac{265}{2268},\frac{3445}{27216}\right).
\]
The full-window errors are $0.06379$, $0.03357$, $0.02098$, $0.01248$, and $0.01024$. All ten controlled main runs take at least one update after actual entry. At $R=80$, G2 enters at $\tau_R=0.51875$ with entry-position error $0.0045922$; its actual continuation covers $S-\tau_R=0.1037403$, compared with the ideal duration $0.1013079$.

\paragraph{Evaluation peak and cumulative coefficients.}
For G2 the evaluation loss is $\ell(a^\top w)$, with
$a=(3,-2,2)/\sqrt{17}$ and weight one; it is excluded from the updates. Its ideal margin vanishes at $r=2/13$ and $r=1255/2592$, and is negative strictly between these times. The most negative value occurs at $r=121/360$, giving
\[
\Gamma_S=\frac{553}{1440\sqrt{17}}\simeq0.0931404171.
\]
Both ideal endpoint margins are positive: $a^\top u_0=1/(2\sqrt{17})$ and $a^\top P(S)\simeq0.0821243$. Figure~\ref{fig:app-numerical-supplement}(a) shows the finite-$R$ maxima. At $R=80$, the unnormalized maximum is $7.71475$, while the losses at the initial and stopping nodes are $6.12\times10^{-5}$ and $0.0010651$. This correction-stage peak is distinct from the continuation-stage analytic peak in Appendix~\ref{app:peak-example}.

The cumulative predictions are obtained directly from the same construction:
\[
\begin{aligned}
A_* &= \int_0^{\tau_*}H(P(r))\,dr
=\frac{\|u_0\|^2-\|E\|^2}{2}
\simeq0.3827732679,\\
S-\tau_*&=0.10130787905092598.
\end{aligned}
\]
Integrating $H$ by the trapezoid rule on each correction segment gives the same value because $H$ is affine on that segment. The ratios in Figure~\ref{fig:numerical-functionals}(b) are $C_R^-/(A_*R^2)$ and $C_R^+/((S-\tau_*)R)$, using the actual entry to split the sums. Each ratio uses its own leading scale from Corollary~\ref{cor:cumulative-loss}.

\subsection{Random instances and stability checks}
\label{app:random-stability}

\paragraph{Generation and fixed windows.}
We generate $40$ instances with five signed samples in dimension three, using one NumPy \texttt{default\_rng} stream with seed $20260913$. Each candidate is generated in this order. Draw $e\sim N(0,I_3)$ and normalize it; draw the five rows of $\widetilde Z$ independently from $N(0,I_3)$; draw $\gamma_i$ independently and uniformly on $[0.3,1]$; and set
\[
z_i=\widetilde z_i+(\gamma_i-e^\top\widetilde z_i)e.
\]
Thus $e^\top z_i=\gamma_i\geq0.3$. Reject a candidate of rank below three. Next draw $p\sim\operatorname{Dirichlet}(2,2,2,2,2)$ and reject it if $\min_i p_i<0.05$. Draw and normalize a new independent Gaussian vector $u_0$, rejecting it if all initial margins are nonnegative. Construct its ideal path, reject a failed construction or an entry time outside $[0.05,3]$, and continue the same random stream until $40$ instances have been retained. The identifiers $00$--$39$ refer to this acceptance order. Every retained instance uses $\eta=1/\Lambda$, $e_R=0$, and the same five values of $R$ as G1 and G2. The window rule is the $0.15$-margin grid rule in Appendix~\ref{app:numerical-methods}. Table~\ref{tab:app-random-config} records the resulting step sizes and windows.

\begin{table}[htbp]
\centering
\small
\caption{Step sizes and fixed windows for all random instances, rounded to twelve decimal places. The generation sequence and grid rule specify their full floating-point values.}
\label{tab:app-random-config}
\begin{tabular}{crrcrr}
\hline
Instance & $\eta$ & $S$ & Instance & $\eta$ & $S$\\
\hline
00 & 2.383566011883 & 0.812550874730 & 20 & 3.239836363412 & 0.687360271366 \\
01 & 4.988250994041 & 3.192085672009 & 21 & 3.359344012717 & 1.983970959559 \\
02 & 4.079246859906 & 3.703531488345 & 22 & 2.326955393044 & 2.986875270919 \\
03 & 3.378190352060 & 2.424884905424 & 23 & 3.180895784073 & 0.292845536333 \\
04 & 1.769894663824 & 0.691047805794 & 24 & 3.970855233719 & 1.771797801208 \\
05 & 1.661883647012 & 2.979381305143 & 25 & 3.614814813406 & 0.875733263005 \\
06 & 2.410180655968 & 1.543820241189 & 26 & 1.313826296203 & 0.229567322949 \\
07 & 2.144367990609 & 3.361451673691 & 27 & 2.213293968037 & 0.342653214758 \\
08 & 3.635353698392 & 1.987541799767 & 28 & 2.372297412199 & 0.716098840947 \\
09 & 3.170971666367 & 1.468272237429 & 29 & 3.438312436521 & 3.025386520144 \\
10 & 1.392798115213 & 1.013943554119 & 30 & 3.322702624270 & 0.699481226831 \\
11 & 1.517461547499 & 1.719897277674 & 31 & 0.912410289175 & 1.457342562573 \\
12 & 3.414517658797 & 1.377518947141 & 32 & 2.229409107419 & 2.593091801222 \\
13 & 1.841063742509 & 1.952046163583 & 33 & 0.790274036819 & 0.450866533256 \\
14 & 2.049682853661 & 0.221885935260 & 34 & 2.191084682065 & 1.736853461839 \\
15 & 5.253370794500 & 2.285533068582 & 35 & 3.192526787799 & 2.384707644311 \\
16 & 2.717695573409 & 2.544737414187 & 36 & 3.043030409953 & 2.446346535484 \\
17 & 9.083183596016 & 2.883648439779 & 37 & 2.402212213338 & 3.127057049348 \\
18 & 1.797200700981 & 2.878985942083 & 38 & 4.438017389159 & 2.095695262535 \\
19 & 4.854169257180 & 2.112908573481 & 39 & 3.647588277199 & 1.929897739430 \\
\hline
\end{tabular}
\end{table}

\paragraph{Coverage and aggregation.}
All $200$ random runs reach their specified windows. At each value of $R$, Figure~\ref{fig:numerical-paths}(c) uses all $40$ errors, with the median and the $10$th and $90$th percentiles computed by linear interpolation between ordered observations. Table~\ref{tab:app-random-coverage} separates reaching $S$ from observing an actual continuation. A run can finish while its actual path is still in correction, or enter exactly at the stopping node. Such runs remain valid full-window path comparisons, but do not supply a positive post-entry cumulative sum. The medians and quantiles summarize the finite-$R$ variation across geometries.

\begin{table}[htbp]
\centering
\small
\caption{All $40$ random runs complete the fixed window at each value of $R$. ``After entry'' counts runs with at least one post-entry update; the next two columns give entry at the stopping node and no entry by that node.}
\label{tab:app-random-coverage}
\begin{tabular}{rrrrrrr}
\hline
$R$ & After entry & Entry at stop & No entry & 10\% & Median & 90\%\\
\hline
5 & 14 & 20 & 6 & 0.07693 & 0.18467 & 0.36983\\
10 & 20 & 15 & 5 & 0.05373 & 0.13004 & 0.21629\\
20 & 26 & 11 & 3 & 0.04282 & 0.09769 & 0.15039\\
40 & 36 & 3 & 1 & 0.03021 & 0.05787 & 0.09460\\
80 & 39 & 1 & 0 & 0.01850 & 0.03701 & 0.06194\\
\hline
\end{tabular}
\end{table}

\paragraph{Initialization perturbations and step sizes.}
We also use $e_R=f(R)v$, with the fixed unit vector $v=(0.48,-0.64,0.60)$ and $f(R)\in\{1,\sqrt R,R^{3/4}\}$, and vary $\eta$ while keeping the geometry and window fixed. Table~\ref{tab:app-stability} reports selected completed checks at $R=80$. For G2 with $f(R)=\sqrt R$, the full-window error equals the initial normalized perturbation to numerical precision, even though the entry-position error is smaller. This illustrates why the theorem retains the initial error separately. Steps with $\eta\Lambda>2$ lie outside the condition for the logarithmic bound, although they are covered by the general fixed-step theorem.

At $R=80$, the G1 runs with $f(R)=\sqrt R$ or $R^{3/4}$ and the G2 run with $f(R)=R^{3/4}$ exhausted their $4\times10^9$-update budgets before reaching $S$. Their partial-window errors are not included in the table or used to estimate a convergence rate.

\begin{table}[htbp]
\centering
\small
\caption{Completed stability checks at $R=80$. The entry error is $\|w_{\kappa_R}/R-E\|$ relative to the unperturbed ideal entry.}
\label{tab:app-stability}
\begin{tabular}{lrrrrr}
\hline
Geometry & $\eta$ & $f(R)$ & $\eta\Lambda$ & $\mathcal E_R(S)$ & Entry error\\
\hline
G1 & 0.25 & 0 & 0.20505 & 0.00795 & 0.00528\\
G1 & 0.5 & 0 & 0.41010 & 0.01160 & 0.00664\\
G1 & 1 & 0 & 0.82019 & 0.00885 & 0.00799\\
G1 & 2 & 0 & 1.64039 & 0.01433 & 0.01383\\
G1 & 4 & 0 & 3.28078 & 0.02796 & 0.02790\\
G1 & 1 & 1 & 0.82019 & 0.01398 & 0.00945\\
G2 & 0.25 & 0 & 0.61538 & 0.00650 & 0.00454\\
G2 & 0.5 & 0 & 1.23075 & 0.01024 & 0.00459\\
G2 & 1 & 0 & 2.46151 & 0.00904 & 0.00725\\
G2 & 2 & 0 & 4.92301 & 0.02080 & 0.01232\\
G2 & 0.5 & 1 & 1.23075 & 0.01256 & 0.00699\\
G2 & 0.5 & $\sqrt{80}$ & 1.23075 & 0.11180 & 0.03430\\
\hline
\end{tabular}
\end{table}

\subsection{Fixed features from handwritten digits}
\label{app:digits-details}

\paragraph{Data and initial direction.}
We use the $1797$ images returned by scikit-learn's \texttt{load\_digits}, corresponding to the test portion of the UCI optical-recognition dataset \citep{alpaydin1998digits}. Each image has $64$ pixel values in $\{0,\ldots,16\}$. We divide pixels by $16$, append a constant coordinate equal to one, and multiply all $65$ coordinates by the binary label. In each ordered pair of digits, the first has label $+1$ and the second label $-1$. The bias coordinate is trained and is included in the Euclidean norm and maximum-margin problem.

For each new task, retain its two digits in dataset order, initialize a fresh NumPy \texttt{default\_rng(0)}, and permute these indices. The first $\lfloor0.7n\rfloor$ form the training sequence and the rest the test set, without stratification. The 32-image task uses the first $32$ members of the $5/6$ training sequence and leaves the $109$-image test set unchanged. Its zero-based indices in the original $1797$-image array, in training order, are
\begin{center}
\small
\begin{tabular}{rrrrrrrr}
1692 & 1320 & 834 & 88 & 272 & 1136 & 1093 & 1769\\
1358 & 1010 & 1172 & 1173 & 969 & 741 & 1292 & 1762\\
1147 & 196 & 1115 & 625 & 717 & 1650 & 720 & 1447\\
1738 & 490 & 351 & 176 & 195 & 1550 & 1035 & 734
\end{tabular}
\end{center}
This smaller training set was added after the initial larger-task runs showed that actual entry delays occupied much of the affordable window at $R\leq80$. This fixed geometry allows both stages to be observed at these values of $R$.

Let $Z_A$ contain all signed features of the source digit pair. We solve
\[
v_A=\operatorname*{arg\,min}_{v\in\operatorname{conv}\{z_i^A\}}\|v\|,
\qquad u_0=v_A/\|v_A\|,
\]
using the convex-hull solver with tolerance $10^{-12}$. This determines the source maximum-margin direction directly; no source-task GD trajectory is used for these initializations. Across the three source tasks, the largest violation of $z_i^{A\top}v_A\geq\|v_A\|^2$ was $8.05\times10^{-16}$, simplex feasibility error was at most $4.45\times10^{-16}$, and complementarity error was at most $5.09\times10^{-17}$. Recomputed directions agreed with the saved directions within $7.67\times10^{-15}$. Separability of each target training set was checked by solving $Z_B a\geq\mathbf1$ with linear programming; the smallest attained margin across the four selected target sets exceeded $1-10^{-14}$.

\paragraph{Complete window for the 32-image task.}
The source pair is $3/8$, comprising all $357$ images. These same images, with uniform weights, define the evaluation loss in Figure~\ref{fig:app-numerical-supplement}(b). The $32$ training weights are $1/32$; ten initial training margins are negative. We use
\[
\eta=0.3178149016002142=1/\Lambda,\quad
\tau_*=0.489846798088871,\quad S=0.5648431242378853.
\]
At each $R\in\{5,10,20,40,80\}$, the run reaches $S$ and includes at least one GD update after entry. The ideal source-evaluation maximum is $0.03560724$, attained at $r=0.49297656$; the ideal endpoint value is $0.03363778$. At $R=80$, the observed normalized maximum is $0.03572499$, attained at $r=0.50679593$. The actual entry time is $0.50055847$, its position error is $0.0011856$, and the full-window error is $0.0195576$. The difference between these two spatial errors reflects, in part, the unshifted entry-time comparison. Evaluation on the held-out target digits is a separate measurement; it is not the source-evaluation curve shown here.

\paragraph{Entry locations for larger tasks.}
Table~\ref{tab:app-digits-tasks} gives the three larger training problems, all using uniform weights and $\eta=1/\Lambda$. Their initial directions are obtained from the complete source pairs by the same procedure. Figure~\ref{fig:app-numerical-supplement}(c) compares actual entry positions with the corresponding ideal entries at $R=10^2,10^3,10^4,10^5$. The errors at $R=10^5$ are respectively $9.95\times10^{-6}$, $4.67\times10^{-5}$, and $2.18\times10^{-5}$. These runs test the entry-location prediction; the figure does not report full continuation-window errors for the larger tasks.

\begin{table}[htbp]
\centering
\small
\caption{Larger digits tasks used for the entry-location comparison. The source margin is $\min_i z_i^{A\top}u_0$. The target split and feature map are specified above.}
\label{tab:app-digits-tasks}
\begin{tabular}{ccrrrrr}
\hline
Source & Target & Source size & Train/test & Source margin & $\eta$ & $\tau_*$\\
\hline
$3/8$ & $5/6$ & 357 & 254/109 & 0.20773 & 0.328481301706 & 6.473229321622\\
$4/9$ & $2/3$ & 361 & 251/109 & 0.37357 & 0.317649989551 & 14.787703491645\\
$0/1$ & $7/9$ & 360 & 251/108 & 0.59168 & 0.342935275073 & 16.407770420857\\
\hline
\end{tabular}
\end{table}

\begin{figure}[htbp]
\centering
\includegraphics[width=\linewidth]{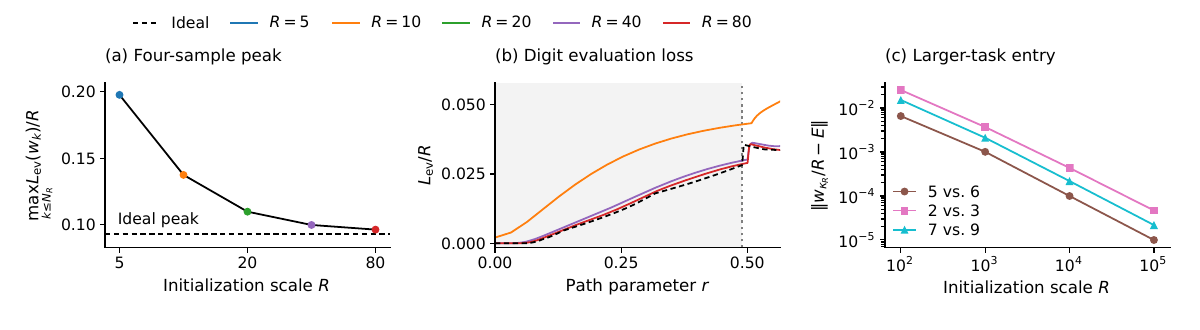}
\caption{\textbf{Additional functional and entry-location results.}
(a) The maximum evaluation loss divided by $R$ in the full-rank four-sample geometry G2 approaches the ideal value $553/(1440\sqrt{17})$; maxima include every GD iterate through the stopping iterate.
(b) Mean loss on all $357$ source images of digits $3/8$, divided by $R$, along GD trained on the fixed $32$ images of digits $5/6$. The black curve is the ideal evaluation profile, and shading marks the ideal correction interval. Colored curves connect saved nodes using each run's actual time parameterization.
(c) Entry-position errors $\|w_{\kappa_R}/R-E\|$ for the three larger target tasks in Table~\ref{tab:app-digits-tasks}. Each run stops at actual entry; this panel measures entry convergence rather than a completed post-entry path.}
\label{fig:app-numerical-supplement}
\end{figure}

\end{APPENDICES}
\end{document}